\documentclass[a4paper,fleqn]{cas-dc}

 \usepackage[numbers]{natbib}

\usepackage[utf8]{inputenc}                         
\usepackage[T1]{fontenc}                            
\usepackage{url}                                    
\usepackage{booktabs}                               
\usepackage{amsfonts, amssymb, amsthm, amsmath}     
\usepackage{ulem}
\usepackage{nicefrac}                               
\usepackage{microtype}                              
\usepackage{fancyhdr}                               
\usepackage{graphicx}                               
\usepackage{bm}                                     
\usepackage{xcolor}                                 
\usepackage{longtable}
\usepackage{supertabular}
\usepackage{paralist}
\usepackage{textcomp}
\usepackage{caption}
\usepackage{subcaption}

\usepackage{soul}

\newcommand{\sms}[1]{\{\!\vert #1 \vert\!\}}
\newcommand{\multiset}[1]{\{\!\vert #1 \vert\!\}}

\newcommand{\set}[1]{\mathcal{#1}}

\newcommand{\straight}[1]{{\upshape #1}}

\def\tsc#1{\csdef{#1}{\textsc{\lowercase{#1}}\xspace}}
\tsc{WGM}
\tsc{QE}
\tsc{EP}
\tsc{PMS}
\tsc{BEC}
\tsc{DE}

\newtheorem{Thm}{Theorem}[subsection]
\newtheorem{NThm}{Theorem}[section]
\newtheorem{Lem}[Thm]{Lemma}
\newtheorem{Prop}[Thm]{Proposition}

\theoremstyle{definition} 

\newtheorem{Def}[Thm]{Definition}
\newtheorem{NDef}[NThm]{Definition}

\newtheorem{NRem}[NThm]{Remark}

\affiliation[1]{Graphs in Artificial Intelligence and Neural Networks (GAIN), University of Kassel, Germany}
    
\affiliation[2]{Siena Artificial Intelligence Lab (SAILab), University of Siena, Italy}

\begin{document}
\let\WriteBookmarks\relax
\def\floatpagepagefraction{1}
\def\textpagefraction{.001}

\shorttitle{Weisfeiler--Lehmann goes Dynamic}

\shortauthors{Beddar-Wiesing et al.}

\title [mode = title]{Weisfeiler--Lehman goes Dynamic: An Analysis of the Expressive Power of Graph Neural Networks for Attributed and Dynamic Graphs}                      

\tnotetext[1]{This work was partially supported by the Ministry of Education and Research Germany (BMBF, grant number 01IS20047A) and partially by INdAM GNCS group.}



%

\author[1]{Silvia Beddar-Wiesing}[orcid=0000-0002-2984-2119]

\cormark[1]
\tnotemark[1]
\fnmark[1]

\ead{s.beddarwiesing@uni-kassel.de}

\ead[url]{www.gain-group.de}

\credit{Conceptualization, Methodology, Formal analysis, Investigation, Writing - Original Draft, Writing - Review \& Editing, Visualization}

\author[2]{Giuseppe Alessio D'Inverno}[orcid=0000-0001-7367-4354]
\ead{giuseppealessio.d@student.unisi.it}
\ead[URL]{https://sailab.diism.unisi.it/people/giuseppe-alessio-dinverno/}
\credit{Conceptualization, Methodology, Formal analysis, Investigation, Software Implementation, Writing - Original Draft, Writing - Review \& Editing}
\fnref{fn1}

\author[2]{Caterina Graziani}[orcid=0000-0002-7606-9405]
\fnref{fn1}
\ead{caterina.graziani@student.unisi.it}
\ead[URL]{https://sailab.diism.unisi.it/people/caterina-graziani/}
\credit{Conceptualization, Methodology, Formal analysis, Investigation, Writing - Original Draft, Writing - Review \& Editing}

\author[2]{Veronica Lachi}[orcid=0000-0002-6947-7304]
\ead{veronica.lachi@student.unisi.it}
\ead[URL]{https://sailab.diism.unisi.it/people/veronica-lachi/}
\credit{Conceptualization, Methodology, Formal analysis, Investigation, Writing - Original Draft, Writing - Review \& Editing}
\fnref{fn1}

\author[1]{Alice Moallemy-Oureh}[orcid=0000-0001-7912-0969]
\ead{amoallemy@uni-kassel.de}
\ead[URL]{www.gain-group.de}
\credit{Conceptualization, Methodology, Formal analysis, Investigation, Writing - Original Draft, Writing - Review \& Editing}
\fnref{fn1}

\author[2]{Franco Scarselli}[orcid=0000-0003-1307-0772]
\ead{franco@diism.unisi.it}
\ead[URL]{https://sailab.diism.unisi.it/people/franco-scarselli/}
\credit{Conceptualization, Formal Analysis,  Writing - Review \& Editing, Project Administration, Supervision}

\author[1]{Josephine Maria Thomas}
\ead{jthomas@uni-kassel.de}
\ead[URL]{www.gain-group.de}
\credit{Writing - Review \& Editing, Supervision, Funding Acquisition}

\fntext[fn1]{These authors contributed equally.}
\cortext[cor1]{Corresponding author}

\begin{abstract}
Graph Neural Networks (GNNs) are a large class of relational models for graph processing. Recent theoretical studies on the expressive power of GNNs have focused on two issues.
On the one hand, it has been proven that  GNNs are as powerful as the Weisfeiler-Lehman test (1--WL) in their ability to distinguish graphs. Moreover, it has been shown that the equivalence enforced by 1-WL equals unfolding equivalence.
On the other hand,  GNNs turned out to be universal approximators on graphs modulo the constraints enforced by 1--WL/unfolding equivalence. However, these results only apply to Static Attributed Undirected Homogeneous Graphs (SAUHG) with node attributes. In contrast, real-life applications often involve a much larger variety of graph types. In this paper, we conduct a theoretical analysis of the expressive power of GNNs for two other graph domains that are particularly interesting in practical applications, namely dynamic graphs  and SAUGHs with edge attributes. Dynamic graphs are widely used in modern applications; hence, the study of the expressive capability of GNNs in this domain is essential for practical reasons and, in addition, it requires a new analyzing approach due to the difference in the architecture of dynamic GNNs compared to static ones. On the other hand, the examination of SAUHGs is of particular relevance since they act as a standard form for all graph types: it has been shown that all graph types can be transformed without loss of information
to SAUHGs with both attributes on nodes and edges. This paper considers generic GNN models and appropriate 1--WL tests for those domains. Then, the known results on the expressive power of GNNs are extended to the mentioned domains: it is proven that GNNs have the same capability as the 1--WL test, the 1--WL equivalence equals unfolding equivalence and that GNNs are universal approximators modulo 1--WL/unfolding equivalence.
Moreover, the proof of the approximation capability is mostly constructive and allows us to deduce hints on the architecture of GNNs that can achieve the desired approximation.
\end{abstract}




\maketitle

\section{Introduction}
Graph data is becoming pervasive in many application domains, such as biology, physics, and social network analysis \cite{skardinga2021foundations,kazemi_survey_dyn_gnn}. Graphs are handy for complex data since they allow for naturally encoding information about entities, their links, and their attributes.
In modern applications, several different types of graphs are commonly used and possibly combined: graphs can be homogeneous or heterogeneous, directed or undirected, have attributes on nodes and/or edges, and be static or dynamic hyper- or multigraphs \cite{thomas2021graph}.
Considering the diversity of graph types, it has recently been shown that Static Attributed Undirected Homogenous Graphs (SAUHGs) with both attributes on nodes and edges can act as a standard form for graph representation, namely that all the common graph types can be transformed into those SAUHGs without losing their encoded information~\cite{thomas2021graph}. 

In the field of graph learning, Graph Neural Networks (GNNs) have become a prominent class of models used to process graphs and address different learning tasks directly. For this purpose, most GNN models adopt a computational scheme based on a local aggregation mechanism. It recursively updates the local information of a node stored in an attribute vector by aggregating the attributes of neighboring nodes. After several iterations, the attributes of a node capture the local structural information received from its $k$--hop neighborhood. At the end of the iterative process, the obtained node attributes can address different graph-related tasks by applying a suitable readout function.

The Graph Neural Network model of \cite{jour_scarselli_2009}, which is called Original GNN (OGNN) in this paper, was the first model capable of facing both node/edge and graph-focused tasks utilizing suitable aggregation and readout functions, respectively. In the current research, a large number of new applications and models have been proposed, including Neural Networks for graphs \cite{micheli2009neural}, Gated Sequence Graph Neural Networks~\cite{ li2015gated}, Spectral Networks~\cite{bruna2013}, Graph Convolutional Neural Networks~\cite{kipf2016}, GraphSAGE~\cite{hamilton2017inductive}, Graph Attention Networks~\cite{GAT}, and Graph Networks~\cite{Battaglia2018}. 
Another extension of GNNs consisted of the proposal of several new models capable of dealing with dynamic graphs~\cite{kazemi_survey_dyn_gnn}, which can be used, e.g., to classify sequences of graphs, sequences of nodes in a dynamic graph or to predict the appearance of an edge in a dynamic graph. 

\medskip

Recently, a great effort has been dedicated to studying the theoretical properties of GNNs \cite{jegelka2022theory}. Such a trend is motivated by the attempt to derive the foundational knowledge required to design reasonable solutions efficiently for many possible applications of GNNs \cite{zhou2020graph}. A particular interest lies in the study of the expressive power of GNN models since, in many application domains, the performance of a GNN depends on its capability to distinguish different graphs. For example, in Bioinformatics, the properties of a chemical molecule may depend on the presence or absence of small substructures; in social network analysis, the count of triangles allows us to evaluate the presence and the maturity of communities. Similar examples exist in dynamic domains: the distinction of substructures contributes significantly to the successful analysis of molecular conformations \cite{luo2021predicting}; in studies of the evolution of social networks \cite{deng2019learning}, substructures in the form of contextual knowledge can help to improve performance.



Formally, a central result on the expressive power of GNNs has shown that GNNs are at most as powerful as the Weisfeiler–Lehman graph isomorphism test (1–WL) \cite{leman1968,wltest2011,xu2018powerful}.
The 1--WL test iteratively divides graphs into groups of possibly isomorphic\footnote{It is worth noting that the 1--WL test is inconclusive since there exist pairs of graphs that the test recognizes as isomorphic even if they are not.} graphs using a local aggregation schema. Therefore, the 1--WL test exactly defines the classes of graphs that GNNs can recognize as non-isomorphic.

Furthermore, it has been proven that the equivalence classes induced by the 1--WL test are equivalent to the ones obtained from the unfolding trees of nodes on two graphs \cite{krebs2015universal} \cite{d2021unifying}, and hence, 1--WL and the unfolding equivalences can be used interchangeably.
An unfolding tree rooted at a node is constructed by starting at the root node and unrolling the graph along the neighboring nodes until it reaches a certain depth. If the unfolding trees of two nodes are equal in the limit, the nodes are called unfolding tree equivalent\footnote{Currently, the concept underlying unfolding trees is widely used to study the GNN expressiveness, even if they are mostly called {\em computation graphs}~\cite{garg2020generalization}.}. Fig.~\ref{fig:diagram} visualizes the relations between the 1--WL, the unfolding tree equivalences, and the GNN expressiveness.

\begin{figure}
    \centering
    \includegraphics[width=\linewidth]{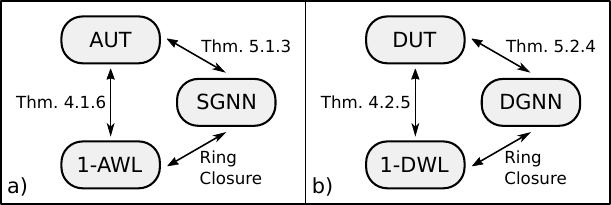}
    \caption{a) In Thm.~\ref{ue=wl}, 
    we prove the equivalence  of the attributed unfolding tree equivalence (AUT)
    and the attributed 1--WL equivalence (1--AWL) for SAUHGs. Afterward in Thm.~\ref{theorem_universal_approx_sauhg}, we show a result on
    the approximation capability of static GNNs for SAUHGs (SGNN) using the AUT equivalence. b) Analogously to the attributed case, we show similar results for Dynamic GNNs (DGNN) 
    which can be used on temporal graphs.}
    \label{fig:diagram}
\end{figure}

\medskip

Another research goal of the expressive power of GNNs is to study the GNN approximation capability. Formally, it has been proven in~\cite{scarselli2008computational} that OGNNs can approximate in probability, up to any degree of precision, any measurable function $\tau(\mathbf{G},v)\rightarrow \mathbb{R}^m$  that respects the unfolding equivalence. Such a result has been recently extended to a large class of GNNs~\cite{d2021unifying} called message-passing GNNs, including most contemporary architectures.

Despite the availability of the mentioned results on the expressive power of GNNs, their application is still limited to undirected static graphs with attributes on nodes. This limitation is particularly restrictive since modern applications usually involve more complex data structures, such as heterogeneous, directed, dynamic graphs and multigraphs.
In particular, the ability to process dynamic graphs is progressively gaining significance in many fields such as social network analysis~\cite{deng2019learning}, recommender systems~\cite{xu2020inductive,rossi2020temporal}, traffic forecasting~\cite{yu2017spatio} and knowledge graph completion~\cite{trivedi2017know,wu2020temp}. Several surveys discuss the usage of dynamic graphs in other application domains~\cite{skardinga2021foundations,thomas2022graph,barros2021survey,longa2023graph,xue2022dynamic}.

Although GNNs are considered universal approximators on the extended domains, it is uncertain which GNN architectures contribute to such universality. Moreover, an open question is how the definition of the 1--WL test has to be modified to cope with novel data structures and whether the universal results fail for particular graph types. 

In this paper,  we propose a study on the expressive power of GNNs for two domains of particular interest, namely dynamic graphs and static attributed undirected homogeneous graphs (SAUHGs) with node and edge attributes. On the one hand, dynamic graphs are interesting from a practical and a theoretical point of view and are used in several application domains \cite{kazemi_survey_dyn_gnn}.
Moreover, dynamic GNN models are structurally different from GNNs for static graphs, and the results and methodology required to analyze their expressive power cannot directly be deduced from existing literature.
On the other hand, SAUHGs with node and edge attributes are interesting because, as mentioned above,  they act as a standard form for several other types of graphs that can all be transformed to SAUGHs~\cite{thomas2021graph}.

\medskip
First, we introduce appropriate versions of the 1--WL test and the unfolding equivalence to construct the fundamental theory for the domains of SAUHGs and dynamic graphs and discuss their relation afterward. 
Then, we consider generic GNN models that can operate on both domains and prove their universal approximation capability modulo the aforementioned 1--WL/unfolding equivalences. More precisely, the main contributions of this paper are as follows.

\begin{itemize}
    \item We present new versions of the 1--WL test and of the unfolding equivalence appropriate for dynamic graphs and SAUHGs with node and edge attributes, and we show that they induce the same equivalences on nodes. Such a result makes it possible to use them interchangeably to study the expressiveness of GNNs.
    \item We show that generic GNN models for dynamic graphs and SAUHGs with node and edge attributes are capable of approximating, in probability and up to any precision, any measurable function on graphs that respects the 1-WL/unfolding equivalence.
    \item The result on approximation capability holds for graphs with unconstrained attributes of reals and target
    functions. Thus, most of the domains used in practical applications are included.
    \item Moreover, the proof is based on space partitioning, which  allows us to deduce information about the GNN architecture that can achieve the desired approximation.
    \item We validate our theoretical results conducting an experimental validation. Our setups show that 1) sufficiently powerful DGNNs can  approximate well dynamic systems that preserve the unfolding equivalence  and 2) using non--universal architectures can lead to poor performances. 

\end{itemize}
The rest of the paper is organized as follows. Section~\ref{related_work}  illustrates the related literature. In section~\ref{section_preliminaries}, the notation used throughout the paper is described, and the main definitions are introduced. In section~\ref{section_equivalences}, we introduce the 1--WL test and unfolding equivalences suitable for dynamic graphs and SAUHGs with node and edge attributes and prove that those equivalences are equal. In section~\ref{section_approximation}, the approximation theorems for GNNs on both graph types are presented. We support our theoretical findings by setting up synthetic experiments in \ref{section_experiments}. Finally, section~\ref{section_conclusion} includes our conclusions and future matter of research. 
All of the proofs are collected in the appendix \ref{section_appendix}.
\section{Related Work}\label{related_work}
In the seminal work \cite{xu2018powerful}, it has been proven that standard GNNs have the same expressive
power as  the 1--WL test.  To overcome such a  limitation, new GNN models and variants of the WL test have been proposed. For example,
in \cite{you2021identity}, a model is introduced where node identities are directly injected into the aggregate functions.
In \cite{morris_2019_WL_go_neural}, the k-WL test has been taken into account to develop a more powerful GNN model, given its greater capability to distinguish non-isomorphic graphs.
In \cite{bodnar2021weisfeiler}, a simplicial-wise Weisfeiler-Lehman test is introduced, and it is proven to be strictly more powerful than the 1-WL test and not less powerful than the 3-WL test; a similar test (called cellular WL test), proposed in \cite{bodnar2021weisfeilerb}, is proven to be more powerful than the simplicial version. 
Nevertheless, all these tests do not deal with edge-attributed graphs and not with dynamic graphs. \\

In \cite{barcelo2022weisfeiler}, the authors consider a GNN model for multi-relational graphs where edges are labeled with types: it is proven that such a model has the same computation capability as a corresponding WL-test. The result is similar to our result on SAUGH, but the way we aggregate the message on the edges is different.  Additionally, our work extends \cite{barcelo2022weisfeiler} from several points of view: the approximation capability of GNNs is studied, a relationship between the 1-WL test and unfolding trees is established, and edge attributes include more general vectors of reals. Moreover, all those studies are extended to dynamic graphs. 


Moreover, the WL test mechanism applied to GNNs has also been studied within the paradigm of \textit{unfolding trees} \cite{sato2020survey}, ~\cite{zhang2021nested}. An equivalence between the two concepts has been established by \cite{krebs2015universal}, but it is limited to static graphs without edge attributes.
 
Some studies have been dedicated  to the approximation and generalization properties of Graph Neural Networks. In \cite{scarselli2008computational}, the authors proved  the universal approximation properties of the original Graph Neural Network model, \textit{modulo} the unfolding equivalence. The universal approximation is shown for GNNs with random node initialization in \cite{abboud2020}, while, in~\cite{xu2018powerful}, GNNs are shown to be able to encode any graph with countable input features. Moreover, the authors of \cite{dasoulas2019coloring} proved that GNNs, provided with a colored local iterative procedure (CLIP), can be universal approximators for functions on graphs with node attributes.
The approximation property has also been extended to Folklore Graph Neural Networks in~\cite{maron2019provably} and Linear Graph Neural Networks, and general GNNs in \cite{azizian2020expressive,maron2018invariant}, both in the invariant and equivariant case. Recently, the universal approximation theorem has been proved for modern message-passing graph neural networks by \cite{d2021unifying} giving a hint on the properties of the network architecture (e.g., the number of layers, the characterization of the aggregation function). 
A relation between the graph diameter and the
computational power of GNNs has been established  in~\cite{loukas2019graph}, where the GNNs are assimilated to the so--called LOCAL models and it is proved that a GNN with a number of layers larger than the diameter of the graph can compute any Turing function of the graph. Nevertheless,  no information on the aggregation function characterization is given.

Despite the availability of universal approximation theorems for static graphs with node attributes, the theory lacks results about the approximation capability of other types of graphs, such as dynamic graphs and graphs with attributes on both nodes and edges. 
Therefore, this paper aims to extend the results of the expressive power of GNNs for dynamic graphs and SAUHGs with node and edge attributes.

\section{Notation and Preliminaries}\label{section_preliminaries}
Before extending the work about the expressive power of GNNs to dynamic and edge-attributed graph domains, 
the mathematical notation and preliminary definitions are given in this chapter. In this paper, only finite graphs are under consideration.
\medskip
\begin{center}
\tablehead{\hline\textbf{Notation} &\\\hline}\label{tab_notation}
\tabletail{\hline}
\begin{tabular}{|l|l|}
	\hline
	$\mathbb{N}$ & natural numbers\\
	$\mathbb{N}_0$ & natural numbers starting at $0$\\
	$\mathbb{R}$ & real numbers\\
	$\mathbb{R}^k$ & $\mathbb{R}$ vector space of dimension $k$\\
	$\mathbb{N}^k$ & $\mathbb{N}$ vector space of dimension $k$\\
	$\mathbb{Z}^k$ & $\mathbb{Z}$ vector space of dimension $k$\\
	\hline
	$\mathbb{0}=(0,\ldots,0)^\top$ & zero vector of corresponding size\\
	$|a|$ & absolute value of a real $a$ \\
	$\|\cdot \|$ & norm on $\mathbb{R}$ \\
	$\|\cdot \|_p$ & $p$-norm on $\mathbb{R}$ \\
	$\|\cdot \|_\infty$ & $\infty$-norm on $\mathbb{R}$ \\
	$|M|$ & number of elements of a set $M$\\ 
	$[n], \,n\in \mathbb{N}$ & sequence $1,2,\ldots, n$ \\
	$[n]_{0}, \,n\in \mathbb{N}_{0}$ & sequence $0,1,\ldots, n$ \\
	\hline
	$\emptyset$ & empty set\\
	$\perp$ & undefined; non-existent element\\
	$\{\cdot\}$ & set \\
	$\{\!\vert\cdot\vert\!\}$ & multiset, i.e. set allowing \\
	& multiple appearances of entries\\
	$(x_i)_{i \in I}$ & vector of elements $x_i$ \\
	& for indices in set $I$\\
	$[v|w]$ & stacking of vectors $v,w$ \\
	\hline
	$\wedge$ & conjunction \\ 
	$\cup$ & union of two (multi)sets\\
	$\subseteq$ & sub(multi)set \\
	$\subsetneq$ & proper sub(multi)set \\
	$\times $ & factor set of two sets\\
	\hline
\end{tabular}
\end{center}

\vspace{0.2cm}

\noindent

The following definition introduces a static, node/edge attributed, undirected and homogeneous graph called SAUHG. The reason for defining and using it comes from \cite{thomas2021graph}. Here, it is shown that every graph type is bijectively transformable into each other.
Therefore, SAUHGs will be used as a standard form for all structurally different graph types as directed or undirected simple graphs, hypergraphs, multigraphs, heterogeneous or attributed graphs, and any composition of those. For a detailed introduction, see \cite{thomas2021graph}.
\begin{NDef}[Static Attributed Undirected Homogeneous Graphs]\label{def_sauhg}
$G'$ is called \textbf{static, node/edge attributed, undirected, homogeneous graph (SAUHG)} if $G'=(\mathcal{V}', \mathcal{E}', \alpha', \omega')$, with $\mathcal{V}' \subset \mathbb{N}$ is a finite set of nodes, ${\mathcal{E}' \subset \{\{u,v\} \mid \forall\, u,v \in \mathcal{V}' \}}$ is a finite set of edges and node and edge attributes are determined by the mappings ${\alpha'\, :\, \mathcal{V}' \rightarrow A,\ \omega'\, :\, \mathcal{E}'\rightarrow B}$ that map into the arbitrary node attribute set $A$, and edge attribute set $B$.
The domain of SAUGHs will be denoted as $\mathcal{G}'$.
\end{NDef}

\begin{NRem}{(Attribute sets)}\label{rem_attribute_sets_equal}
    In the above definition of the SAUHG, the node and edge attribute sets $A$ and $B$ can be arbitrary. However, without loss of generality, we can assume that the attribute sets are equal because they can be arbitrarily extended to $A' = A \cup B$. Additionally, any arbitrary attribute set $A'$ can be embedded into the $k$-dimensional vector space $\mathbb{R}^k$. Since the attribute sets in general do not matter for the theories in this paper and to support a better readability, in what follows we consider ${\alpha'\, :\, \mathcal{V}' \rightarrow \mathbb{R}^k,\ \omega\, :\, \mathcal{E}'\rightarrow \mathbb{R}^k}$ for every SAUHG.
\end{NRem}

\noindent All the aforementioned graph types are static, but temporal changes play an essential role in learning on graphs representing real-world applications; thus, dynamic graphs are defined in the following. 
In particular, the dynamic graph definition here consists of a discrete-time representation. 


\begin{NDef}[Dynamic Graph]\label{def_discrete_dynamic_graph}
Let $I=[0, \ldots, l]\subsetneq  \mathbb{N}_0$ be a set of timesteps. 
Then a \textbf{(discrete) dynamic graph} can be considered as a vector of static graph snapshots, i.e., $G = (G_t)_{t\in I}$, where $G_{t} = (\mathcal{V}_t, \mathcal{E}_t) \; \forall t \in I$. 
Furthermore, 
\begin{align*}
    \alpha_{v}(t) := \alpha(v,t), &\quad v \in \mathcal{V}_t 
    \quad \text{and} \quad 
\end{align*}
\begin{align*}
    \omega_{\{u,v\}}(t) := \omega(\{u,v\},t), &\quad \{u,v\} \in \mathcal{E}_t
\end{align*}

where $\alpha : \mathcal{V} \times I  \rightarrow A\;$  and $\omega : \mathcal{E} \times I \rightarrow B \;$ define the vector of \textbf{dynamic node/edge attributes}. Further, let $\mathcal{V}:= \bigcup_{t\in I}\mathcal{V}_t$ and $\mathcal{E}:= \bigcup_{t\in I}\mathcal{E}_t$ the total node and edge set of the dynamic graph.


In particular, when a node $v$ does not exist, its attributes and neighborhood are empty, respectively. 
Moreover, let 
\begin{align*}
    \Omega_{ne_v}(t) &= \left({\;{\omega_{\{v, x_1\}}(t), \ldots,  \omega_{\{v, x_{|ne_v(t)|}\}}(t)\; }  \;}\right)_{t\in I} 
\end{align*}
be the sequence of \textbf{dynamic edge attributes of the neighborhood}  corresponding node at each timestep. Note that, as in Rem.~\ref{rem_attribute_sets_equal}, in what follows, we assume the attribute sets to be equal and corresponding to $\mathbb{R}^k$.

\end{NDef}



\begin{figure*}[!ht]
    \centering
    \includegraphics [width=.8\linewidth]{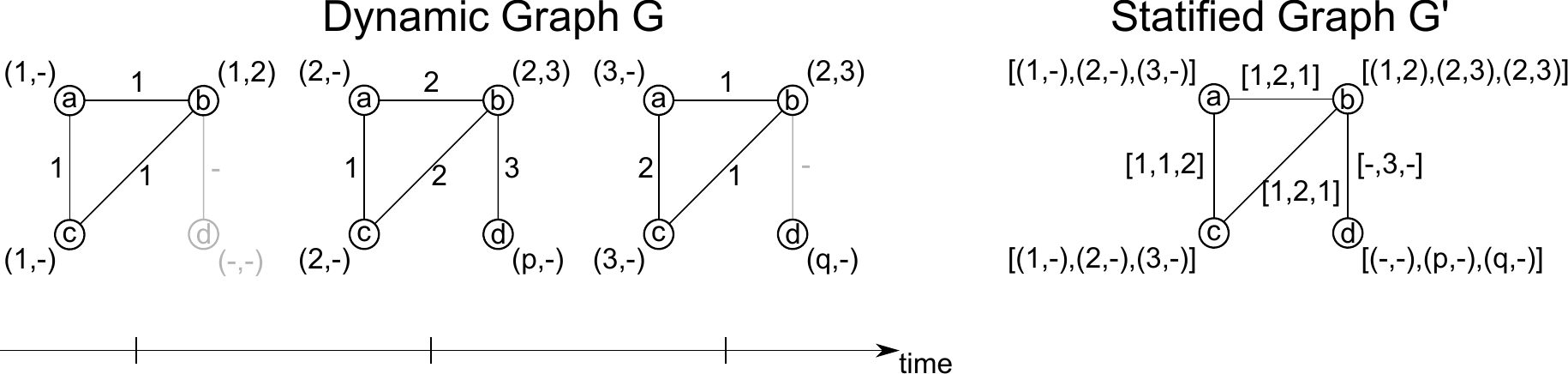}
    \caption{Illustration of the statification of a dynamic graph. On the left, the temporal evolution of a graph, including non-existent nodes and edges (gray), is given, and on the right, the corresponding statified graph with the total amount of nodes and edges together with the concatenated attributes is shown.}
    \label{fig:makestatic}
\end{figure*}
To prove the approximation theorems for SAUHGs and dynamic graphs, we need to specify the GNN architectures capable of handling those graph types.
The standard Message-Passing GNN for static node-attributed undirected homogeneous graphs is given in \cite{xu2018powerful}. Here, the node attributes are used as the initial representation and input to the GNN. The update is executed by aggregation over the representations of the neighboring nodes. 

Given that a SAUHG acts as a standard form for all graph types, the ordinary GNN architecture  will be extended to take also edge attributes into account. This can be done by analogously including the edge attributes in the first iteration to the processing of the node information in the general GNN framework as follows.
\begin{NDef}[SGNN]\label{def:SGNN}
For a SAUHG $G'=(\mathcal{V}', \mathcal{E}', \alpha', \omega')$ let $u,v \in \mathcal{V}'$ and $e = \{u, v\}$.
The SGNN propagation scheme for iteration $i\in [l],\ l > 0$ is defined as
{\small\begin{equation*}
    \mathbf{h}_v^{(i)} = \text{COMBINE}^{(i)}\left( \mathbf{h}_v^{(i-1)}, \text{AGGREGATE}^{(i)}\left(\{ \mathbf{h}_u^{(i-1)}\}_{ u \in ne_v}, \{\omega_{\{u,v\}}\}_{ u \in ne_v}  \right)\right)
\end{equation*}}
The output for a node-specific learning problem after the last iteration $l$ respectively is given by
\begin{equation*}
    \mathbf{z}_v = \text{READOUT}\left( \mathbf{h}_v^l\right),
\end{equation*}
using a selected aggregation scheme and a suitable READOUT function, and the output for a graph-specific learning problem is determined by
\begin{align*}
    \mathbf{z} = \text{READOUT}\left( \{\mathbf{h}_v^l \mid v \in \set{V}'\}\right).
\end{align*}
\end{NDef}

For the dynamic case, we chose a widely used GNN model that is consistent with the theory we built.
Based on \cite{skardinga2021foundations}, the discrete dynamic graph neural network (DGNN) uses a GNN to encode each graph snapshot. Here, the model is modified by using the previously defined SGNN in place of the standard one.
\begin{NDef}[Discrete DGNN]\label{DGNN}
Given a discrete dynamic graph $G = (G_t)_{t\in I}$, a \textbf{discrete DGNN} using a continuously differentiable recursive function $f$ for temporal modelling can be expressed as:
\begin{equation}\label{eq_dgnn}
    \begin{split}
        \mathbf{h}_1(t), \dots, \mathbf{h}_n(t) & := \text{SGNN}(G_{t}) \quad \forall\, t \geq 0\\
        \mathbf{q}_1(0), \dots , \mathbf{q}_n(0) & =  \mathbf{h}_1(0), \dots, \mathbf{h}_n(0) := \text{SGNN}(G_{0})\\        
        \mathbf{q}_v(t) & :=  f(\mathbf{q}_v(t-1), \mathbf{h}_v(t)) \quad \forall\, v \in \mathcal{V}
    \end{split}
\end{equation}
where $\mathbf{h}_v(t) \in \mathbb{R}^{r}$ is the hidden representation of node $v$ at time $t$ of dimension $r$ and $\mathbf{q}_v(t)\in\mathbb{R}^{s}$ is an $s$-dimensional hidden representation of node $v$ produced by $f$, 
and ${f:\mathbb{R}^{s} \times \mathbb{R}^{r}\rightarrow \mathbb{R}^{s}}$ is a neural architecture for temporal modeling (in the methods surveyed in \cite{skardinga2021foundations}, $f$ is almost always an RNN or an LSTM). 

The stacked version of the discrete DGNN is then:
\begin{equation}\label{stack_eq}
    \begin{split}  
        H(t) & = \text{SGNN}(G_{t}) \\
        Q(0) & = H(0) = \text{SGNN}(G_{0}) \\
        Q(t) & = F(Q(t-1),H(t))
    \end{split}
\end{equation}
where ${H(t)\in \mathbb{R}^{n\times r}}$, ${Q(t)\in \mathbb{R}^{n\times s}}$, ${F:\mathbb{R}^{n\times s} \times \mathbb{R}^{n\times r}\rightarrow \mathbb{R}^{n \times s}}$, being $n$ the number of nodes, $r$ and $s$ the dimensions of the hidden representation of a node produced respectively by the $\text{SGNN}$ and by the $f$. Applying $F$ corresponds to component-wise applying $f$ for each node \cite{skardinga2021foundations}.

To conclude, a function $\text{READOUT}_{\text{dyn}}$ will take as input $Q(t)$ and gives a suitable output for the considered task, so that altogether the DGNN will be described as
\begin{equation*}
    \varphi(t,G,v) = \text{READOUT}_{\text{dyn}}(Q(t)).
\end{equation*}

\end{NDef}


\begin{NRem}
    The DGNN is a Message-Passing model because the SGNN is one, by definition. 
\end{NRem}
The GNNs expressivity is studied in terms of their capability to distinguish two non-isomorphic graphs.


\begin{NDef}[Graph Isomorphism]\label{def_graph_iso}
Let $G_1$ and $G_2$ be two static graphs, then $G_1=(\set{V}_1, \set{E}_1)$ and $G_2=(\set{V}_2, \set{E}_2)$ are \textbf{isomorphic} to each other $G_1 \approx G_2$, if and only if there exists a bijective function $\phi: \mathcal{V}_1 \rightarrow \mathcal{V}_2$, with 
 \begin{compactenum}
     \item $v_1 \in \mathcal{V}_1 \, \Leftrightarrow \, \phi(v_1) \in \mathcal{V}_2\ \; \forall \; v_1 \in \mathcal{V}_1$,
     \item ${\{v_1, v_2\} \in \mathcal{E}_1 \, \Leftrightarrow \, \{\phi(v_1), \phi(v_2)\} \in \mathcal{E}_2 \; \forall \, \{v_1, v_2\} \in \mathcal{E}_1}$
 \end{compactenum}

\noindent
In case the two graphs are \textbf{attributed}, i.e., $G_1=(\set{V}_1, \set{E}_1, \alpha_1, \omega_1)$ and $G_2=(\set{V}_2, \set{E}_2, \alpha_2, \omega_2)$, then $G_1 \approx G_2$ if and only if additionally there exist bijective functions $\varphi_{\alpha} : A_1 \rightarrow A_2$ and $\varphi_{\omega} : B_1 \rightarrow B_2$ with images $A_i := \text{im}(\alpha_i)$ and $B_i := \text{im}(\omega_i)$, $i = 1, 2$. 
\hspace*{-0.4cm}
\begin{compactenum}
\item $\varphi_{\alpha}(\alpha_1(v_1)) = \alpha_2(\phi(v_1)) \quad \forall \; v_1 \in \mathcal{V}_1$,
\item ${\varphi_{\omega}(\omega_1(\{u_1, v_1\})) = \omega_2(\{\phi(u_1), \phi(v_1)\}) \; \forall \; \{u_1, v_1\} \in \mathcal{E}_1}$.
\end{compactenum}

\noindent
    If the two graphs are \textbf{dynamic}, they are called to be isomorphic if and only if the static graph snapshots of each timestep are isomorphic. 
\end{NDef}

Graph isomorphism (GI) gained prominence in the theory community when it emerged as
one of the few natural problems in the complexity class NP that could neither be classified as being hard (NP-complete) nor shown to be solvable with an efficient algorithm (that is, a polynomial-time algorithm) \cite{grohe2020graph}. Indeed it lies in the class NP-Intermidiate. 
However, in practice, the so-called Weisfeiler-Lehman (WL) test is used to at least recognize non-isomorphic graphs \cite{sato2020survey}. If the WL test outputs two graphs as isomorphic, the isomorphism is likely but not given for sure. 

\medskip
The expressive power of GNNs can also be approached from the point of view of their approximation capability.\\
It generally analyzes the capacity of different GNN models to approximate arbitrary functions \cite{bronstein2021geometric}. 
Different universal approximation theorems can be defined depending on the model, the considered input data, and the sets of functions. This paper will focus on the set of functions that preserve the unfolding tree equivalence, defined as follows.



Since the results in this section hold for undirected and unattributed graphs, we aim to extend the universal approximation theorem to GNNs working on SAUHGs (cf.~Def.~\ref{def_sauhg}) and dynamic graphs (cf.~Def.~\ref{def_discrete_dynamic_graph}). For this purpose, in the next sections, we introduce a static attributed and a dynamic version of both the WL test and the unfolding trees to show that the graph equivalences regarding the attributed/dynamic WL test and attributed/dynamic unfolding trees are equivalent. With these notions, we define the set of functions that are attributed/dynamic unfolding tree preserving and reformulate the universal approximation theorem to the attributed and dynamic cases (cf.~Thm.~\ref{theorem_universal_approx_sauhg} and Thm.~\ref{dyn_thm_approximation}). 
\section{Weisfeiler-Lehman and Unfolding Trees}\label{section_equivalences}

There are many different extensions of the WL test, e.g., n-dim.~WL test, n-dim folklore WL test, or set n-dim.~WL test \cite{sato2020survey}. 
In general, the 1-WL test is based on a graph coloring algorithm. The coloring is applied in parallel on the nodes of the two input graphs, and at the end, the number of colors used per each graph is counted. Then, two graphs are detected as non-isomorphic if these numbers do not coincide, whereas when the numbers are equal, the graphs are possibly isomorphic (WL equivalent).
Another way to check the isomorphism of two graphs and therefore compare the GNN expressivity is to consider the so-called unfolding trees of all their nodes. An unfolding tree consists of a tree constructed by a breadth-first visit of the graph, starting from a given node. 
Two graphs are possibly isomorphic if all their unfolding trees are equal. From \cite{krebs2015universal}, it is known that for static undirected and unattributed graphs, both the unfolding tree and the Weisfeiler-Lehman approach for testing the isomorphism of two graphs are equivalent. 



In this section, we extend the unfolding tree (UT) and the Weisfeiler-Lehman (WL) tests to the domains of SAUHGs and dynamic graphs respectively. Using these  we show that also the extended versions of UT-equivalence and  WL-equivalence are equivalent.

\subsection{Equivalence for Attributed Static Graphs}\label{section_equivalence_SAUHG}

The extended result on SAUHGs is formalized and proven in Thm.~\ref{ue=wl}. 
The original WL test and unfolding tree notions cover all graph properties except edge attributes. Thus, the notion of unfolding trees and the Weisfeiler-Lehman test have to be extended to an attributed version.  

\begin{Def}[Attributed Unfolding Tree]\label{def_unfolding_tree_attributed}
The \textbf{attributed unfolding tree} $\mathbf{T}_{v}^d$ in graph $G' = (\mathcal{V}', \mathcal{E}', \alpha', \omega')$ of node $v \in \mathcal{V}'$ \textbf{up to depth} $d \in \mathbb{N}_0$ is defined as
\begin{align*}
    \mathbf{T}_{v}^d = 
    \begin{cases}
        Tree(\alpha'_{v}), \quad \text{if } d = 0 \\
        Tree\bigl(\alpha'_{v}, \Omega'_{ ne_v} , \mathbf{T}^{d-1}_{ne_v}\bigr) \quad \text{if } d > 0\,,
    \end{cases}
\end{align*}
where ${Tree(\alpha'_{v})}$ is a tree constituted of node $v$ with attribute $\alpha'_v$. $Tree\bigl(\alpha'_{v}, \Omega'_{ ne_v} , \mathbf{T}^{d-1}_{ne_v}\bigr)$ is the tree consisting of the root node $v$ and subtrees $\mathbf{T}^{d-1}_{ne_v} = \sms{\mathbf{T}_{u_1}^{d-1}, \ldots, \mathbf{T}_{u_{|ne_v|}}^{d-1}}$ of depth $d-1$, that are connected by the corresponding edge attributes $\Omega'_{ne_v} = \sms{\omega'_{\{v, u_1\}}, \ldots,  \omega'_{\{v, u_{|ne_v|}\}}}$ of the neighbors of $v$.

\begin{figure}[ht!]
\begin{center}
  \def\svgwidth{.5\linewidth}
  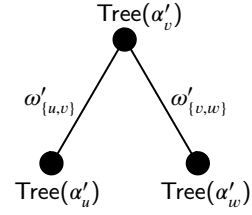
  \caption{Unfolding tree recursive construction}
    \label{fig:my_label}
    \end{center}
\end{figure}
Moreover, the \textbf{attributed unfolding tree} of $v$ determined by $\mathbf{T}_v = \lim_{d \rightarrow \infty} \mathbf{T}_v^d$ is obtained by merging all unfolding trees $\mathbf{T}^d_v$ of any depth $d$.
\end{Def}

\begin{Def}[Attributed Unfolding Equivalence]\label{def_attributed_unfolding_tree_equivalence}

Let $G_1 = (\mathcal{V}_1, \mathcal{E}_1, \alpha_1, \omega_1)$ and $G_2 = (\mathcal{V}_2, \mathcal{E}_2, \alpha_2, \omega_2)$ be two SAUHGs. Then $G_1$ and $G_2$ are \textbf{attributed unfolding tree equivalent}, noted by $G_1 \sim_{AUT} G_2$, if and only if ${\multiset{\mathbf{T}_u \mid u \in \mathcal{V}_1} = \multiset{\mathbf{T}_v \mid v \in \mathcal{V}_2}}$.
Analogously, two nodes $u\in {V}_1, v\in \set{V}_2$ are unfolding tree equivalent, noted by $u \sim_{AUT} v$ if and only if $\mathbf{T}_u = \mathbf{T}_v$.
\end{Def}

Using the definition of the attributed unfolding equivalence on graphs, the 1-Weisfeiler Lehman (1-WL) test provided in \cite{d2021unifying} is extended to attributed graphs. This allows for the definition of the attributed 1-WL equivalence on graphs and the subsequent Lem.~\ref{equal_trees=WL_coloring} and the resulting Thm.~\ref{ue=wl} that pose the relation between the attributed unfolding equivalence and the attributed 1-WL test.

\begin{Def}[Attributed 1-WL test]\label{def_WL_test_attributed}
Let HASH be a bijective function that codes every possible node attribute with a color from a color set $\mathcal{C}$ and $G' = (\mathcal{V}', \mathcal{E}', \alpha', \omega')$.
The \textbf{attributed 1-WL (1-AWL) test} is defined recursively through the following.
\begin{itemize}
    \item At iteration $i=0$, the color is set to the hashed node attribute:
    \begin{align*}
        c_v^{(0)} = \text{HASH}(\alpha'_v) 
    \end{align*}
    \item At iteration $i>0$, the HASH function is extended to the edge weights:
    \begin{align*}
        c_v^{(i)} = \text{HASH}\Bigl( \bigl( c_v^{(i-1)}, \Omega'_{ ne_v}, c_{ne_v}^{(i-1)}\bigr)\Bigr)
    \end{align*}
\end{itemize}
\end{Def}

In the following, the 1-WL equivalence of graphs and nodes is extended by using the attributed version of the 1-WL test (cf.~Def.~\ref{def_WL_test_attributed}).

\begin{Def}[Attributed 1-WL equivalence]\label{def_attributed_wl_equivalence}
Two nodes $u, v$ are attributed WL-equivalent, noted by $u \sim_{AWL} v$, if and only if $c_u = c_v$. \\Analogously,
let $G_1 = (\mathcal{V}_1, \mathcal{E}_1, \alpha_1, \omega_1)$ and $G_2 = (\mathcal{V}_2, \mathcal{E}_2, \alpha_2, \omega_2)$ be two SAUHGs. Then, $G_1 \sim_{AWL} G_2$, if and only if for all nodes $v_1 \in \mathcal{V}_1$ there exists a corresponding node $v_2 \in \mathcal{V}_2$ such that $v_1 \sim_{AWL} v_2$.
\end{Def}

Finally, to complete the derivation of the equivalence between the attributed unfolding equivalence (cf.~Def.~\ref{def_attributed_unfolding_tree_equivalence}) and the attributed WL-equivalence (cf.~Def.~\ref{def_attributed_wl_equivalence}), the following helping lemma Lem.~\ref{equal_trees=WL_coloring} is given, which directly leads to Thm.~\ref{ue=wl}. 
The lemma states the equivalence between the attributed unfolding tree equivalence of nodes and the equality of their attributed unfolding trees up to a specific depth. 
In \cite{scarselli2008computational}, it has been shown that the unfolding trees of infinite depth are not necessary to consider for this equivalence. Instead, the larger number of nodes  
of both graphs under consideration is sufficient for the depth of the unfolding trees, which is finite since the graphs are bounded. 
The following lemma determines the equivalence between the attributed unfolding trees of two nodes and their colors resulting from the attributed 1-WL test. 

\begin{Lem}\label{equal_trees=WL_coloring}
Consider $G' = (\mathcal{V}', \mathcal{E}' , \alpha', \omega')$ as the SAUHG resulting from a transformation of an arbitrary static graph $G = (\mathcal{V}, \mathcal{E} , \alpha, \omega)$ with nodes $u, v \in \mathcal{V}$ and corresponding attributes $\alpha_u, \alpha_v$. Then it holds 
\begin{align*}
    \forall \; d \in \mathbb{N}_0 : \mathbf{T}^{d}_u = \mathbf{T}^{d}_v \Longleftrightarrow c_u^{(d)} = c_v^{(d)}.
\end{align*}
\end{Lem}

The proof can be found in Apx.~\ref{proof_of_equal_trees=WL_coloring}

Directly from Lem.~\ref{equal_trees=WL_coloring},  the equivalence of the attributed unfolding tree equivalence and the attributed 1-WL equivalence of two nodes belonging to the same graph can be formalized.


\begin{Thm}\label{ue=wl}
Consider $G'$ as in Lem.~\ref{equal_trees=WL_coloring}. 
Then, it holds
\begin{align*}
    \forall \; u, v \in \mathcal{V}' : u \sim_{AUT} v \Longleftrightarrow u \sim_{AWL} v.
\end{align*} 
\end{Thm}

\begin{proof}
The proof follows from the proof of  Lem.~\ref{equal_trees=WL_coloring}. 
\end{proof}


\subsection{Equivalence for Dynamic Graphs} \label{section_equivalence_dyn_graphs}
In this section, the previously introduced concepts of unfolding tree and WL equivalences are extended to the dynamic case. Note that Lem.~\ref{equal_trees=WL_coloring} and, therefore, Thm.~\ref{ue=wl} also hold in case $G'$ is the SAUHG (cf.~Def.~\ref{def_sauhg}) resulting from a transformation of a dynamic graph $G = (G_t)_{t\in I}$ to its static attributed version. 
However, the GNNs working on dynamic graphs usually use a significantly different architecture than those that work on static attributed graphs. Therefore, the following includes the derivation of the various equivalences on dynamic graphs separately.

First, dynamic unfolding trees are introduced as a sequence of unfolding trees for each graph snapshot respectively. Afterward, the equivalence of two dynamic graphs regarding their dynamic unfolding trees is presented. 



\begin{Def}[Dynamic Unfolding Tree] \label{def_unfolding_tree_dynamic} Let $G = (G_t)_{t\in I}$ with $G_t = (\mathcal{V}_t, \mathcal{E}_t, \alpha_t, \omega_t)$ be a dynamic graph.
The \textbf{dynamic unfolding tree} $\mathbf{T}_{v}^d(t)$ at time $t\in I$ of node $v\in\mathcal{V}$ up to depth $d \in \mathbb{N}_0$ is defined as
\begin{align*}
    \mathbf{T}_{v}^d (t)= 
    \begin{cases}
         Tree(\alpha_{v}(t)), \quad \text{if } d = 0
        \\
        Tree\bigl(\alpha_{v}(t), \Omega_{ ne_{v}(t)} , \mathbf{T}^{d-1}_{ne_{v}(t)}(t)\bigr) \quad \text{if } d > 0,
    \end{cases}
    \end{align*}

where ${Tree(\alpha_{v}(t))}$ is a tree constituted of node $v$ with attribute $\alpha_{v}(t)$. Furthermore, $Tree\bigl(\alpha_{v}(t), \Omega_{ ne_{v}(t)} , \mathbf{T}^{d-1}_{ne_{v}(t)}(t)\bigr)$ is the tree with root node $v$ with attribute $\alpha_{v}(t)$. Additionally, $\mathbf{T}^{d-1}_{ne_{v}(t)}(t) = \sms{\mathbf{T}_{u_1}^{d-1}(t), \ldots, \mathbf{T}_{u_{|ne_{v}(t)|}}^{d-1}(t)}$ are corresponding subtrees with edge attributes $\Omega_{ ne_{v}(t)}$ . If the node $v$ does not exist at time $t$, the corresponding tree is empty, 
there is no tree of depth $d>0$ for this timestep and $v$ does not occur in any neighborhood of other nodes.

In total, the \textbf{dynamic unfolding tree of $v$} at time $t$,  ${\mathbf{T}_{v}(t) = \lim \limits_{d \rightarrow \infty} \mathbf{T}_{v}^d(t)}$ is obtained by merging, i.a., all the unfolding trees $\mathbf{T}_{{v}}^d(t)$ for any $d$.
\end{Def}


\begin{Def}\label{def:unf_eq}
Two nodes $u,\ v\in\mathcal{V}$ are said to be \textbf{dynamic unfolding equivalent} $u\sim_{DUT} v$ if $\mathbf{T}_{u}(t)= \mathbf{T}_{v}(t)$ for every timestep t.  
Analogously,
two dynamic graphs $G_1, \; G_2$ are said to be \textbf{dynamic unfolding equivalent} $G_1 \sim_{DUT} G_2$, if there exists a bijection between the nodes of the graphs that respects the partition induced by the unfolding equivalence on the nodes.
\end{Def}

\begin{Def}[Dynamic 1-WL test]\label{def_WL_test_dynamic}Let $G = (G_t)_{t\in I}$ with $G_t = (\mathcal{V}, \mathcal{E}, \alpha_t, \omega_t)$ be a dynamic graph.
Let $\text{HASH}^0_t$  be a bijective function encoding every node attribute of $G_t$ with a color from a color set $\mathcal{C}$. 
The \textbf{dynamic 1-WL test (1-DWL)} generates a vector of color sets one for each timestep $t\in I$ by:
\begin{itemize}
    \item At iteration $i=0$ the color is set to the hashed node attribute or a fixed color for non-existent nodes:
    \begin{align*}
        c_{v}^{(0)}(t) =\begin{cases} \text{HASH}^0_t\left(\alpha_{v}(t)\right) & \text{if }v\in\mathcal{V}_t,\\
             c^\perp & \text{otherwise.}
        \end{cases}
    \end{align*}
    
    \item Then, the aggregation mechanism is defined by the bijective function $\text{HASH}_t$ for $i>0$:
    \begin{align*}
        c_{v}^{(i)}(t) = \text{HASH}_t\Bigl( \bigl( c_{v}^{(i-1)}(t), \Omega_{ne_v(t)} , c_{ne_v(t)}^{(i-1)}(t)\bigr)\Bigr) 
    \end{align*}
    Note that for $i>0$, $c_v^{(i-1)(t)}=c^\perp$ holds for a non-existent node at time $t$. Further, the neighborset is empty so the other inputs of HASH$_t$ are empty, and together with $c^\perp$ it will always give the same color for non-existent nodes.
\end{itemize}

\end{Def}


\begin{Def}[Dynamic 1-WL equivalence]\label{def_dynamic_wl_equivalence}
Two nodes $u,v\in \mathcal{V}$ in a dynamic graph $G$ are said to be \textbf{dynamic WL equivalent}, noted by $u \sim_{DWL} u$, if their colors resulting from the WL test are pairwise equal per timestep. Analogously, let $G_1$ and $G_2$ be dynamic graphs. Then $G_1 \sim_{DWL} G_2$, if and only if for all nodes $v_1 \in \mathcal{V}^{(1)}_t$ there exists a corresponding node $v_2 \in \mathcal{V}^{(2)}_t$ with $c_{v_1}(t) = c_{v_2}(t)$ for all $t\in I$.

\end{Def}

\begin{Thm}[Equivalence of Dynamic WL Equivalence and Dynamic UT Equivalence for nodes]\label{1-DWL=DUT}
   Let $G= (G_t)_{t\in I}$ be a dynamic graph and $u,v \in \mathcal{V}$.
   Then, it holds
\begin{align*}
     u \sim_{DUT} v  \Longleftrightarrow u \sim_{DWL} v.
\end{align*} 
\end{Thm}
\begin{proof}
    Two nodes are dynamic unfolding tree equivalent iff they are attributed unfolding tree equivalent at each timestep $u \sim_{AUT} v \quad \forall t \in I$ \ref{def:unf_eq}. Further, as consequence of Thm.~\ref{ue=wl}, it holds that for all $t \in I$ the two nodes are attributed WL equivalent $u \sim_{AWL} v$ and thus, the two nodes are dynamic WL equivalent by Def.~\ref{def_dynamic_wl_equivalence}. 
    In case of the non-existence of $u$ at a certain timestep $t$, the Theorem still holds.

    
\end{proof}

    

\section{Approximation Capability of GNN's}\label{section_approximation}

In this section, the results from Sec.~\ref{section_equivalence_SAUHG}
and Sec.~\ref{section_equivalence_dyn_graphs} are brought together in the formulation of a universal approximation theorem for GNNs working on SAUHGs and dynamic graphs and the set of functions that preserve the attributed or dynamic unfolding equivalence, respectively. 

\subsection{GNNs for Attributed Static Graphs}

Since the goal is to show the attributed extension of the universal approximation theorem, it is necessary to define the corresponding family of attributed unfolding equivalence-preserving functions. A function preserves the attributed unfolding equivalence if the output of the function is equal when two nodes are attributed unfolding equivalent. 

\begin{Def}\label{def_preserve_attr_unf_eq}Let $\mathcal{G}'$ be the domain of bounded SAUHGs, $G'=(\mathcal{V}', \mathcal{E}', \alpha', \omega')\in\mathcal{G}'$ a SAUHG and $u,v\in \mathcal{V}'$ two nodes. Then a function $f: \mathcal{G}' \rightarrow \mathbb{R}^m$ is said to \textbf{preserve the attributed unfolding equivalence} on $\mathcal{G}'$ if
\begin{equation*}
v \sim_{AUT} u \Rightarrow f(G',v)=f(G', u).
\end{equation*}
All functions that preserve the attributed unfolding equivalence are  collected in the set $\mathcal{F}(\mathcal{G}')$.
\end{Def}

Analogous to the argumentation in \cite{scarselli2008computational}, there exists a relation between the unfolding equivalence preserving functions and the unfolding trees for attributed graphs, as follows.

\begin{Prop}[Functions of attributed unfolding trees] \label{f_unfold}
A function $f$ belongs to $\mathcal{F}(\mathcal{G}')$ if and only if there exists a function $\kappa$ defined on trees such that for any graph $G'\in\mathcal{G}'$ it holds $f(G',v)= \kappa(\mathbf{T}_v)$, for any node $v \in G'$.
\end{Prop}

The proof works analogously to the proof of the unattributed version presented in \cite{scarselli2008computational} and can be found in Apx.~\ref{apx:prop_functions_of_attr_unf_trees}.

Considering the previously defined concepts and statements for SAUHGs in Sec.~\ref{section_equivalence_SAUHG}, finally, the following theorem states the universal approximation capability of the SGNNs on bounded SAUHGs.

\begin{Thm}[Universal Approximation Theorem by SGNN]\label{theorem_universal_approx_sauhg}
Let $\mathcal{G}'$ be the domain of bounded SAUHGs  with the maximal number of nodes ${N= \max\limits_{G' \in \mathcal{G}'} |G'|}$.
For any measurable function $f \in \mathcal{F}(\mathcal{G}')$ preserving 
the attributed unfolding equivalence (cf.~Def.~\ref{def_preserve_attr_unf_eq}), any norm $\| \cdot \|$ on $\mathbb{R}$, any probability measure $P$ 
on $\mathcal{G}'$, for any reals $\epsilon, \lambda$ where $\epsilon, \lambda >0$, there exists a SGNN defined by the continuously differentiable functions $\text{COMBINE}^{(i)}$, $\text{AGGREGATE}^{(i)}$, at iteration $ i \leq 2N-1$, and by the function
$\text{READOUT}$, with hidden dimension $r=1$, i.e,  $h_v^i\in\mathbb{R}\ \forall i$, such that the function
$ \varphi$  (realized by the GNN) computed after $2N-1$ steps for all $G'\in\mathcal{G}'$
satisfies the condition
\begin{equation*}
P( \| f(G',v)- \varphi( G',v) \| \leq \varepsilon) \geq 1- \lambda.
\end{equation*}
\end{Thm}

The corresponding proof can be found in \ref{apx:proof_apx_thm_attributed}. \\

As in \cite{scarselli2008computational} we want now to study the case when the employed 
components (COMBINE, AGGREGATE, READOUT) are sufficiently general to be able to approximate any function preserving the unfolding equivalence.
We call this class of networks, $\cal{Q}_S$, \textit{SGNN models with universal components}. To simplify our discussion, we introduce the transition function $f^{(i)}$  to indicate the concatenation of the $\text{AGGREGATE}^{(i)}$ and $\text{COMBINE}^{(i)}$, i.e.,
\begin{equation*} 
\hspace*{-0.8cm}
\begin{split}
    & \quad \quad f^{(i)}(\mathbf{h}_v^i,\{\mathbf{h}^{i-1}_u, \; u \in ne[v]\} \{\omega_{\{u,v\}}\}_{ u \in ne_v} ) = \\  
    & \quad \text{\footnotesize COMBINE}^{(i)}\left( \mathbf{h}_v^{(i-1)}, \text{\footnotesize AGGREGATE}^{(i)}\left(\{ \mathbf{h}_u^{(i-1)}\}_{ u \in ne_v}, \{\omega_{\{u,v\}}\}_{ u \in ne_v}  \right)\right)
\end{split}
\end{equation*}
Then, we can formally define the class $\cal{Q}_S$.

\begin{Def} \label{def:universal_static}
A class $\cal{Q}_S$  of SGNN models is said to have \textit{universal components} if, for any $\epsilon>0$
and any continuous target functions $\overline{\text{COMBINE}}^{(i)}$, $\overline{\text{AGGREGATE}}^{(i)}$,  $\overline{\text{READOUT}}$, 
 there exists an SGNN belonging to $\cal{Q}_S$, with functions $\text{COMBINE}_\theta^{(i)}$, $\text{AGGREGATE}_\theta^{(i)}$, $\text{READOUT}_\theta$ and
parameters $\theta$  such that
\begin{equation*}
\left\|{\bar f}^{(i)}(\mathbf{h},\{ \mathbf{h}_1,\ldots,\mathbf{h}_n\}) - f_\theta^{(i)}(\mathbf{h},\{ \mathbf{h}_1,\ldots,\mathbf{h}_n\}) 
\right\|_\infty\leq \epsilon
\end{equation*}
\begin{equation*}
\left\| \overline{\text{READOUT}}( \mathbf{q})-\text{READOUT}_\theta( \mathbf{q})\right\|_\infty\leq \epsilon\,,
\end{equation*}
holds, for any vectors  $\mathbf{h}$, $\mathbf{h}_1,\ldots,\mathbf{h}_n\in \mathbb{R}^r$, $\mathbf{q}\in\mathbb{R}^s$.
The transition functions ${\bar f}^{(i)}$ and $f_\theta^{(i)}$ correspond to the target function and the SGNN, respectively.
\end{Def}

\vspace{.3cm}

The following result shows that Theorem~\ref{theorem_universal_approx_sauhg} still holds even for SGNNs with universal components.

\begin{Thm}[Approximation by Neural Networks]\label{mainNN}
	Assuming that the hypotheses of Theorem~\ref{theorem_universal_approx_sauhg} are fulfilled and 
	$\cal{Q}_S$ is a class of  SGNNs with universal components.
	Then, there exists a parameter set $\theta$ and some functions 
	$\text{COMBINE}^{(i)}_\theta$, $\text{AGGREGATE}^{(i)}_\theta$, $\text{READOUT}_\theta$,  implemented 	by Neural Networks in $\cal{Q}_S$, such that the thesis of	Theorem~\ref{theorem_universal_approx_sauhg} holds. 
\label{nntheo_static}
\hfill \qed
\end{Thm}

\begin{proof}
The proof is identical to the one contained in \cite{scarselli2008computational}; to give a hint on the methodology, we refer to the more complex proof of the analogous Theorem \ref{dyn_mainNN} for DGNN.
\end{proof}

\subsection{GNNs for Dynamic Graphs}
Suitable functions that preserve the unfolding equivalence on dynamic graphs are dynamic systems. Before this statement is formalized and proven in Prop.~\ref{f_unfold_dyn}, dynamic systems and their property to preserve the dynamic unfolding equivalence are defined in the following.

\begin{Def}[Dynamic System]\label{def:dyn_sys}
Let $\mathcal{G}$ be a domain of dynamic graphs and let $\mathcal{V}=\bigcup_t \mathcal{V}_t$.
A \textbf{dynamic system} is defined as a function ${\text{\straight{dyn}}: \mathcal{D}:= I \times \mathcal{G} \times  \mathcal{V} \rightarrow \mathbb{R}^m}$ formalized for 
$G = (G_t)_{t \in I} \in \mathcal{G},$  and $v \in \mathcal{V}_t$ by
\begin{equation}\label{dynsis}
    \text{\straight{dyn}}(t,G,v): = g(x_v(t)).
\end{equation}
Here, $g:\mathbb{R}^{r} \rightarrow \mathbb{R}^m$ is an output function, and the \textit{state function} $x_v(t)$ is determined by 
\begin{equation*}
    x_v(t)\; = \;
    \begin{cases}
      a(t,G, v)  &  \text{if}\; t=0 \\
      f(x_v(t-1), a(t-1,G,v))  & \text{if} \; t>0,\\
    \end{cases}
\end{equation*}

\noindent for $v\in\mathcal{V}_t$, where $a:I \times \mathcal{G} \times  \mathcal{V} \rightarrow \mathbb{R}^r$ is a function that processes the graph snapshot at time $t$ and provides an $r$-dimensional internal state representation for each node $v$. Finally, $f:\mathbb{R}^{r}\times \mathbb{R}^r \rightarrow \mathbb{R}^{r}$ is a recursive function, that is called \textit{state update function}.
\end{Def}

\begin{Def}\label{def:sys_unfold}
A dynamic system $\text{\straight{dyn}}(\cdot,\cdot,\cdot)$ \textbf{preserves the dynamic unfolding tree equivalence} on $\mathcal{G}$ if and only if for any input graph sequences $G_1,G_2\in \mathcal{G}$, 
and two nodes $u, v\in\mathcal{V}$
it holds 
\begin{align*}
    v \sim_{DUT} u  \Longrightarrow \text{\straight{dyn}}(t,G_1,v)=\text{\straight{dyn}}(t,G_2,u) \quad\forall t.
\end{align*}
\end{Def}

The class of dynamic systems that preserve the unfolding equivalence on $\mathcal{D}$ will be denoted with $\mathcal{F}(\mathcal{D})$. A characterization of $\mathcal{F}(\mathcal{D})$ is given by the following result (following the work in \cite{scarselli2008computational}).

\begin{Prop}[Functions of dynamic unfolding trees] \label{f_unfold_dyn}
A dynamic system \straight{dyn} belongs to $\mathcal{F}(\mathcal{D})$ if and only if there exists a function $\kappa$ defined on attributed trees such that for all ${(t,G,v)\in \mathcal{D}}$ it holds 
\begin{align*}
\text{\straight{dyn}}(t,G,v)= \kappa\Bigl(\bigl(\mathbf{T}_{v}(i)\bigr)_{i \in [t]}\Bigr).
\end{align*}
\end{Prop}
The proof can be found in Apx.~\ref{proof_of_f_unfold}.

Finally, the universal approximation of the Message-Passing GNN for dynamic graphs is determined as follows.

\begin{Thm}[Universal Approximation Theorem by DGNN]\label{dyn_thm_approximation}
Let  $G = (G_t)_{t\in I}$ be a discrete dynamic graph in the graph domain $\mathcal{G}$ and ${N= \max_{G \in \mathcal{G}} |G|} $ be the maximal number of nodes in the domain.
Let $\text{\straight{dyn}}(t,G,v) \in \mathcal{F}(\mathcal{D})$ be any measurable dynamical system preserving 
the unfolding equivalence,  $\| \cdot \|$ be a norm  on $\mathbb{R}$, $P$ be any probability measure 
on $\mathcal{D}$ and $\epsilon, \lambda$ be any real numbers where $\epsilon,\lambda >0$. Then, there exists a DGNN composed by SGNNs with $2N-1$ layers and hidden dimension $r=1$, and Recurrent Neural Network state dimension
$s=1$ such that the function $ \varphi$ realized by this model satisfies 
\begin{equation*}
P( \| \text{\straight{dyn}}(t,G,v)- \varphi(t, G,v) \| \leq \varepsilon) \geq 1- \lambda \quad \quad \quad \forall t \in I.
\end{equation*} 
\end{Thm}

The proof can be found in Apx.~\ref{proof_approx}.\\
Theorem~\ref{dyn_thm_approximation} intuitively states that, given a dynamical system $\text{dyn}$, there is a DGNN that approximates it. 
The functions which the DGNN is a composition of (such as the dynamical function $f$, $\text{COMBINE}^{(i)}$, $\text{AGGREGATE}^{(i)}$, etc.) are supposed to be continuously differentiable but still generic, while  can be generic and completely unconstrained. This situation does not correspond to practical cases where the DGNN adopts particular architectures, and those functions are Neural Networks, or more generally, parametric models -- for example, made of layers of sum, max, average, etc.
Thus, it is of fundamental interest to clarify whether the theorem still holds when the components of the DGNN are parametric models. 

\begin{Def} \label{def:universal_dyn}
A class $\cal{Q}_D$  of discrete DGNN models is said to have \textit{universal components} if the employed  SGNNs have universal components as defined in  Def.~\ref{def:universal_static} 
and the employed recurrent model is designed such that for any $\epsilon_1, \epsilon_2 >0$
and any continuously differentiable target functions
$\overline{\text{f}}$, $\overline{\text{READOUT}}_{\text{dyn}}$ 
 there is a discrete DGNN in  the class $\cal{Q}_D$, with functions 
 $f_\theta$, $\text{READOUT}_{\text{dyn},\theta}$  and parameters $\theta$  such that, for any input vectors 
 $\mathbf{h}\in\mathbb{R}^r$, $\mathbf{q},\mathbf{q}^\star\in\mathbb{R}^s$, it holds
\begin{equation*}
\begin{split}
\left\| \overline{f}( \mathbf{q}, \mathbf{h})-f_\theta( \mathbf{q},\mathbf{h})\right\|_\infty &\leq \epsilon_1, \; \\
\left\| \overline{\text{READOUT}}_{\text{dyn}}( \mathbf{q}^\star)-\text{READOUT}_{\text{dyn},\theta}( \mathbf{q}^\star)\right\|_\infty &\leq \epsilon_2.
\end{split}
\end{equation*}
\end{Def}


The following result shows that Theorem~\ref{dyn_thm_approximation} still holds even for discrete DGNNs with universal components.

\begin{Thm}[Approximation by Neural Networks]\label{dyn_mainNN}
	Assume that the hypotheses of Thm.~\ref{dyn_thm_approximation} are fulfilled and 
	$\cal{Q}_D$ is a class of discrete DGNNs with universal components.
	Then, there exists a parameter set $\theta$, and the functions 
	$\overline{\text{f}}$, $\overline{\text{READOUT}}_{\text{dyn}}$ ,  implemented by Neural Networks in $\cal{Q}_D$, such that Thm.~\ref{dyn_thm_approximation} holds. 
\end{Thm}

The proof can be found in Apx.~\ref{proof_dyn-mainNN}

\subsection{Discussion}\label{subsec:discussion}
The following remarks may further help to understand the results proven in the previous paragraphs:
\begin{itemize}
     \item  Thm.~\ref{theorem_universal_approx_sauhg} suggests an alternative approach to process 
    several graph domains  with a universal SGNN model. Actually, almost all the graphs, including,
    e.g., hypergraphs, multigraphs, directed graphs, etc., can be transformed to SAUGHs with
    node and edge attributes \cite{thomas2021graph}. Then, we can use a universal GNN model on such a domain using 
    sufficiently expressive AGGREGATE and COMBINE functions.

    
    \item The proofs of Thms.~\ref{theorem_universal_approx_sauhg} and ~\ref{dyn_thm_approximation} are based
    on space partitioning reasoning. Differently from technique based on Stone--Weierstrass theorem \cite{azizian2020expressive}, which
    are existential in nature, such an approach allows us to deduce information about the characteristics
    of networks that reach the desired approximation. Actually, the theorems point out that the approximation
    can be obtained with a minimal hidden dimension $r=1$  both in SGNNs  and DGNNs and with a state dimension $s=1$ in the Recurrent Neural Network of DGNNs. Such a result may appear surprising, 
    but the proofs show that  GNNs can encode unfolding trees with a single real number. 
   
   \item Moreover, Thms.~\ref{theorem_universal_approx_sauhg} and ~\ref{dyn_thm_approximation} specify that  GNNs can obtain
    the approximation with $2N-1$ layers. We could incorrectly presume that the maximum number of layers required to reach a desired approximation depends on the diameter $diam(G)$ of the graph, which can be smaller than the number of nodes $N$ since a GNN can move the information from one node to another in $diam(G)$ iterations. However, $diam(G)$ layers are not always sufficient to distinguish all the nodes of a graph.
    In fact, it has been proven that $N-1$ is a lower bound 
    on the number of iterations that the 1--WL algorithm has to carry out to be able  to distinguish any pairs of 1--WL distinguishable graphs \cite{kiefer2020iteration}, and $2N-1$ is a lower bound for 1--WL algorithm to distinguish pairs of nodes in two different graphs \cite{krebs2015universal}. So overall, $2N-1$ is also the lower bound for the GNN computation time to approximate any function for either graph- or node-focused tasks (see \cite{d2021unifying} for a detailed discussion).
    
   \item   Thms.~\ref{theorem_universal_approx_sauhg} and \ref{dyn_thm_approximation} specify that the approximation is modulo  unfolding equivalence, or, correspondingly, modulo  WL equivalence. 
   It can be observed that in the dynamic case, only a part of the architecture affects the equivalence. Actually, a dynamic GNN contains two modules: the first one, an SGNN, produces an embedding of the input graph at each time instance; the second component contains a Recurrent Neural Network that processes the sequence of the embeddings. The dynamic  unfolding equivalence is defined by sequences of unfolding trees,  which are built independently for each node and time instance by the SGNN. Similarly, the dynamic  WL equivalence is defined by  sequences of colors defined independently at each time step. Intuitively, the Recurrent Neural Network does not affect 
   the equivalence, since Recurrent Neural Networks can be universal approximators and  implement
   any function of the sequence without introducing other constraints beyond those already
   introduced by the SGNN.
   
    \item   Thm.~\ref{dyn_thm_approximation} does not hold for any Dynamic GNN, as we take into account a discrete recurrent model working on graph snapshots (also known as Stacked DGNN). Nevertheless, several DGNNs of this kind are listed in \cite{kazemi_survey_dyn_gnn}, such as  GCRN-M1 \cite{seo2018structured}, RgCNN \cite{narayan2018learning}, PATCHY-SAN \cite{niepert2016learning}, DyGGNN \cite{taheri2019learning}, and others. 
     Still, the approximation capability depends on the functions AGGREGATE and COMBINE designed for each GNN working on the single snapshot and the implemented Recurrent Neural Network. For example, the most general model, the original RNN, has been proven to be a universal approximator \cite{hammer2000approximation}.
\end{itemize}

\section{Experimental Validation}\label{section_experiments}
In this Section, we support our theoretical findings with an experimental study. For this purpose, we carry out two sets of experiments described as follows:
\begin{itemize}
    \item[\textbf{E1.}] We show that a DGNN with universal components can approximate a function $F_{DWL}:\mathcal{G}\rightarrow{\mathbb{N}}$ that models the 1-DWL test. The function $F_{DWL}$ assigns a target label to the input graph that represents the class of 1-DWL equivalence;
    \item[\textbf{E2.}] In the same approximation task, we compare DGNNs with different GNN modules from the literature to show how the universality of the components affects the approximation capability.
\end{itemize}
We focus on the ability of the DGNN to approximate $F_{DWL}$, so only training performances are considered, i.e., we do not investigate the generalization capabilities over a test set. 
Since the 1-DWL test provides the finest partition of graphs reachable by a DGNN, the mentioned tasks experimentally evaluate the expressive power of DGNNs.

\paragraph{Dataset.}
The dataset consists of dynamic graphs, i.e., vectors of static graph snapshots of fixed length $T$. Each static snapshot is one of the graphs in Fig.~\ref{fig:dynamic_graph_generation}. Since the dataset is composed of all the possible combinations of the four graphs, it contains  $4^T$ dynamic graphs. Given that the graphs in Fig.~\ref{fig:dynamic_graph_generation} are pairwise 1-WL equivalent ( a) is 1-WL equivalent to b) and c) is 1-WL equivalent to d) ), the number of classes is $2^T$, with $\frac{4^T}{2^T} = 2^T$ graphs in each class. 
For each dynamic graph, the target is the corresponding 1-DWL output, represented as a natural number. For training purposes, the targets are normalized between 0 and 1 and uniformly spaced in the interval $[0,1]$.
Therefore, the distance between each class label is $d=1/2^T$. A dynamic graph $G$ with target $y_G$ will be said to be correctly classified if, given $\mathsf{out} = \text{DGNN}(G)$, we have $|\mathsf{out}-y_G|< d/2 $.

\begin{figure}
\centering
    \includegraphics[width=.8\linewidth]{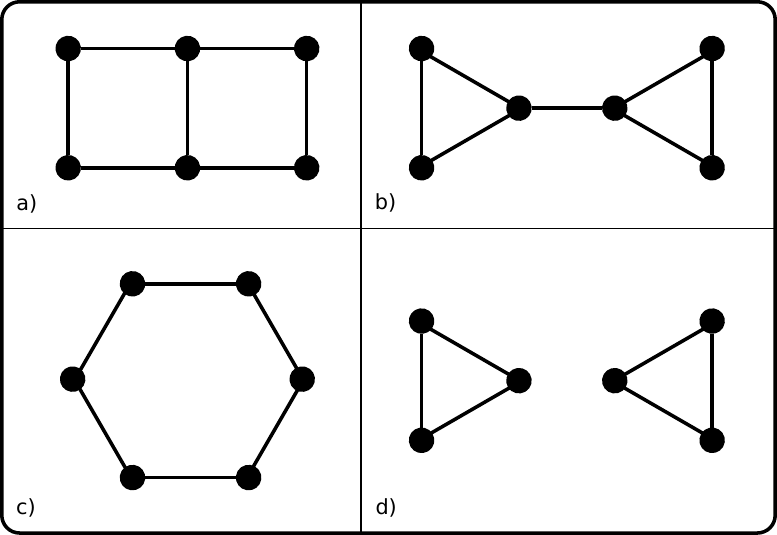}
    \caption{The four static graphs are used as components to generate the synthetic dataset. Graphs a) and b) are equivalent under the static 1--WL test; the same holds for c) and d). } 
    \label{fig:dynamic_graph_generation}
\end{figure}

\paragraph{Experimental setup.}
\begin{itemize}
 \item[E1] For the first set of experiments, the Dynamic Graph Neural Network  is composed of two modules: A Graph Isomorphism Network (GIN) \cite{xu2018powerful} and a Recurrent Neural Network (RNN), which implement the static GNN and the temporal Network $f$ of Eq.~(\ref{eq_dgnn}), respectively. 
Since it has been proven that the GIN is a universal architecture \cite{xu2018powerful} and the RNNs are universal approximators for dynamical systems on vector sequences \cite{hammer2000approximation}, the architecture used in the experiments 
fits the hypothesis of Thm.~\ref{dyn_thm_approximation}. Thus, it can approximate any dynamical system on the temporal graph domain. 

The model hyperparameters for the experiments are set as follows.
The GIN includes $n_{\max}=6$ layers\footnote{As investigated in Subsection \ref{subsec:discussion}, for graph-focused tasks, it is sufficient to perform the message-passing convolution for several times equal to the maximum number of nodes over the graphs in the dataset domain.}. The MLP in the GIN network
contains one hidden layer with a hyperbolic tangent activation function and batch normalization. Hidden layers of different sizes, i.e., $h_{\mathsf{gin}} \in \{1,4,8\}$,
have been tested. For sake of simplicity, the output network has one hidden layer with the same number of neurons as the MLP in the GIN. 
Furthermore, $h_{\mathsf{rnn}}=8$ is the size of the hidden state of the RNN.
\item[E2] 
In the other set of experiments, we test DGNNs composed by different GNN static modules and an RNN module (analogously to \textbf{E1.}). In particular, we compare DGNNs with the GNN module taken from the following list:
\begin{itemize}
    \item GIN  as mentioned before;
    \item Graph Convolutional Network (GCN)  \cite{kipf2016};
    \item GNN presented in \cite{hamilton2017representation} (see also \cite{morris_2019_WL_go_neural}) where the aggregation function is the \textit{sum} of the hidden features of the neighbours; it will be called $\mathsf{gconv\_add}$;
    \item GNN presented in \cite{hamilton2017representation} with \textit{mean} of the hidden features of the neighbours as aggregation function, called $\mathsf{gconv\_mean}$ here; 
    \item GAT  \cite{GAT}.
\end{itemize}
Here,  the used hyperparameters are hidden dimension is $h=8$, the number of layers \mbox{$L=4$}, and the time length \mbox{$L=T=5$}.
\end{itemize}

In both the experimental cases, the model is trained over $300$ epochs using the Adam optimizer with a learning rate of $\lambda = 10^{-3}$. Each configuration is 
is evaluated over $10$ runs. The overall training is then performed on an Intel(R) Core(TM) i7-9800X processor running at 3.80GHz using 31GB of RAM and a GeForce GTX 1080 Ti GPU unit.
The code used to run the experiments can be found at \url{https://github.com/AleDinve/dyn-gnn}.

\paragraph{Results.}
The results of the experiments confirm our theoretical statements. More precisely, the DGNNs performed as follows during training.
\begin{itemize}
    \item[E1] 
\begin{figure*}[ht!]
         \begin{subfigure}{0.49\linewidth}\label{pic:T_5_1}
        \includegraphics[width=\textwidth]
        {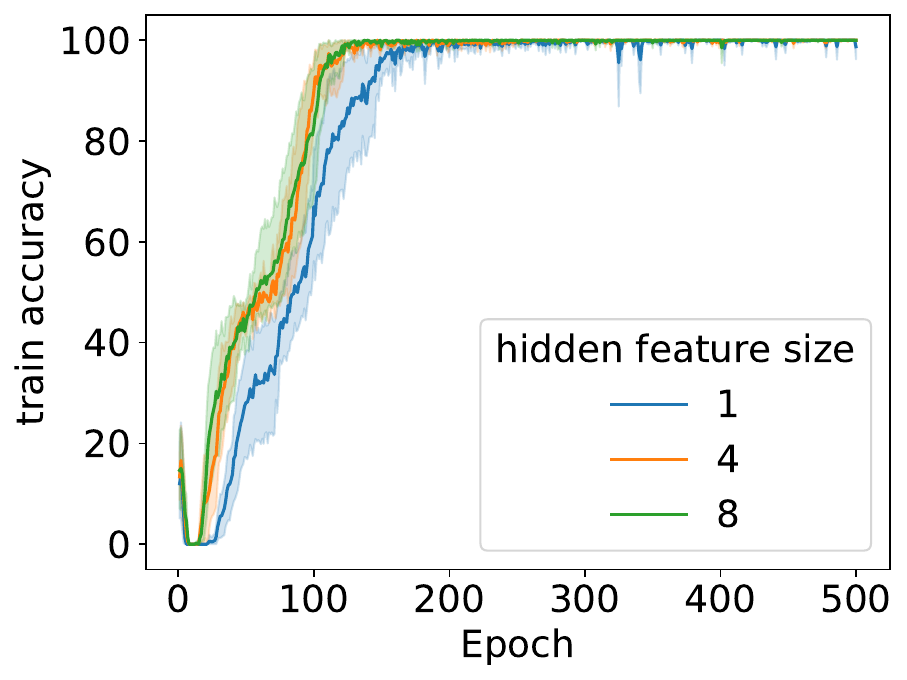} \subcaption{}\label{pic:T_5_1.a}     
    \end{subfigure}
    \begin{subfigure}{0.49\linewidth}\label{pic:T_5_2}
        \includegraphics[width=\textwidth]{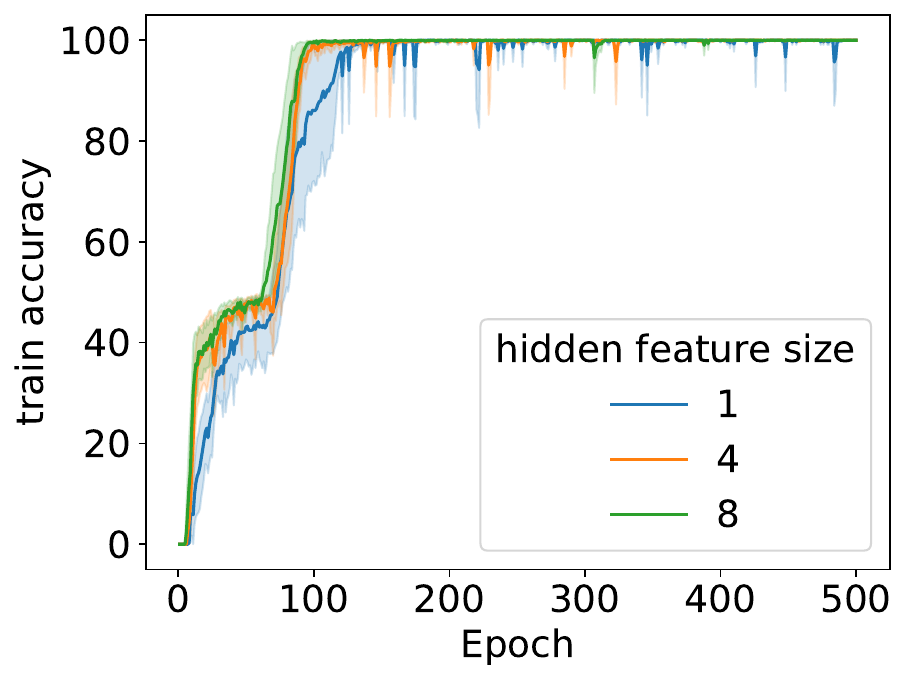}\subcaption{}\label{pic:T_5_1.b}
    \end{subfigure}
     \caption{Experimental Framework E1. Training accuracy over the epochs for a DGNN trained on the dataset containing dynamic graphs up to time length $T=4$ (a) and $T=5$ (b).}
     \label{pic:experiments_acc}   
\end{figure*}
\begin{figure*}[ht!]
         \begin{subfigure}{0.49\linewidth}
        \includegraphics[width=\textwidth]
        {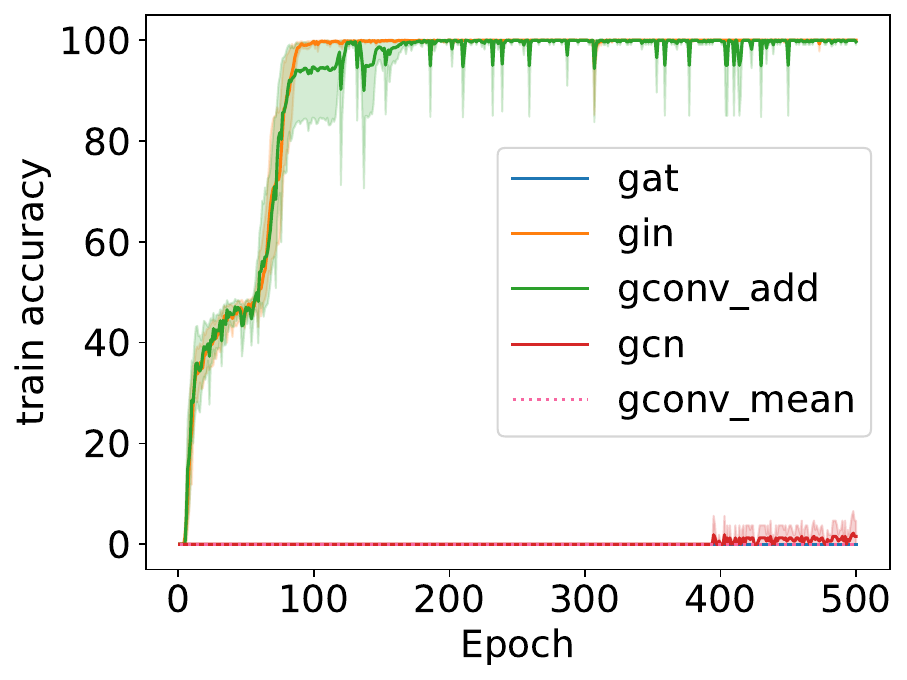} \subcaption{} \label{pic:dgnn_comp1} 
    \end{subfigure}
    \begin{subfigure}{0.49\linewidth}
        \includegraphics[width=\textwidth]{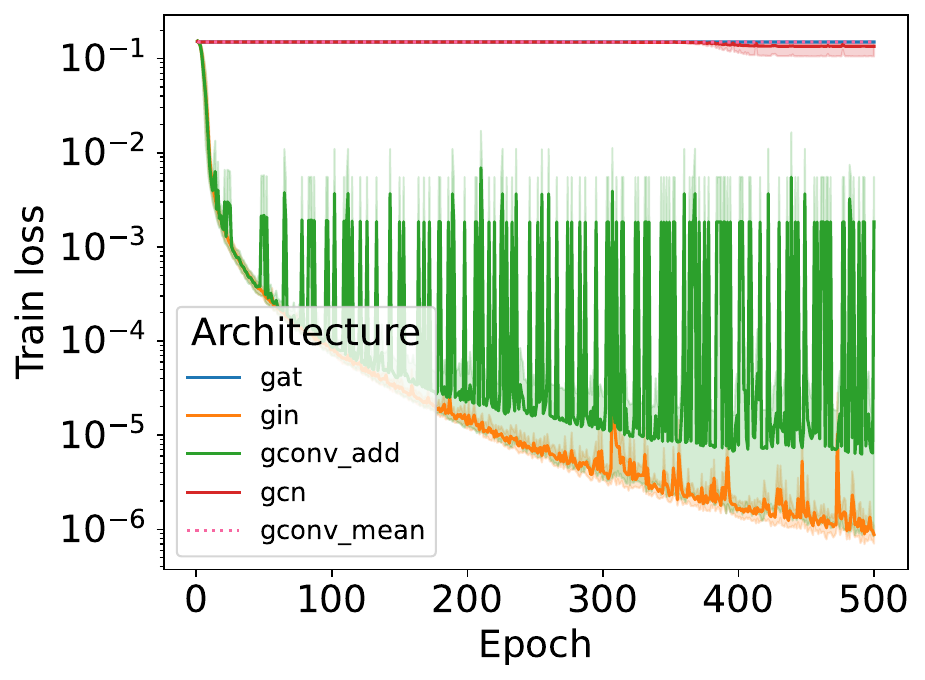}\subcaption{}\label{pic:dgnn_comp2} 
    \end{subfigure}
     \caption{Experimental framework E2. Training accuracy a) and training loss b) over the epochs for several DGNNs trained on the dataset containing dynamic graphs up to time length $T=5$. Figure b) is in logarithmic scale. }
     \label{pic:dgnn_comparison}      
\end{figure*}
In Fig.~\ref{pic:experiments_acc}, the evolution of the training accuracy over the epochs is presented for different GIN hidden layer sizes 
$h_{\mathsf{gin}}$ and for dynamic graphs up to time lengths $T=4$ (Fig.~\ref{pic:T_5_1.a}) and $T=5$ (Fig.~\ref{pic:T_5_1.b}). All the architectures  reach $100\%$ accuracy for experiments on both time lengths. Even setting $h_{\mathsf{gin}}=1$ leads to a perfect classification at a slower rate. It may appear surprising that, even with a hidden representation of size $1$, the DGNN can well approximate the function $F_{DWL}$. 

However, as we already pointed out in Sec.~\ref{subsec:discussion}, the possibility of reaching the universal approximation with a feature of dimension $1$ is confirmed by Thm.~\ref{dyn_thm_approximation}.
\item[E2] 
The DGNN with the GIN module achieve the best performance in terms of learning accuracy and speed of  decreasing, as illustrated in Fig.~\ref{pic:dgnn_comparison}. The DGNN with the $\mathsf{gconv\_add}$ module is able to learn the task, although  learning  is unstable (see Fig.~\ref{pic:dgnn_comparison} b)). This is not surprising since this module has been proven to match the expressive power of the 1-WL test \cite{morris_2019_WL_go_neural}. The other DGNNs are incapable to learn the objective function. This is a consequence of their weaker expressive power, widely investigated in literature \cite{xu2018powerful,d2021unifying}.
\end{itemize}
Thus, overall, our theoretical expectations were met by both experiments.

\section{Conclusion and Future Work}\label{section_conclusion}

This paper provides two extensions of the 1-WL isomorphism test and the definition of unfolding trees of nodes and graphs. First, we introduced WL test notions to attributed and dynamic graphs, and second, we introduced extended concepts for unfolding trees on attributed and dynamic graphs. 
Further, we proved the existence of a bijection between the dynamic 1-WL equivalence of dynamic graphs and the attributed 1-WL equivalence of their corresponding static versions that are bijective regarding their encoded information. The same result we proved w.r.t. the unfolding tree equivalence. 
Moreover, we extended the strong connection between unfolding trees and the (dynamic/attributed) 1–WL tests — proving that they give rise to the same equivalence between nodes for the attributed and the dynamic case, respectively. 
Note that the GNNs working on static graphs usually have another architecture than those working on dynamic graphs. Therefore, we have proved that both the different GNN types can approximate any function that preserves the (attributed/dynamic) unfolding equivalence (i.e., the (attributed/dynamic) 1-WL equivalence). 

\medskip
Note that the dynamic GNN considered here in this paper is given in discrete-time representation, i.e., as a sequence of static graph snapshots without actual timestamps. Thus, Thm.~\ref{dyn_thm_approximation} does not hold for any Dynamic GNN, as we take into account a discrete recurrent model working on graph snapshots (also known as Stacked DGNN). Nevertheless, several DGNNs of this kind are listed in \cite{kazemi_survey_dyn_gnn}, such as  GCRN-M1 \cite{seo2018structured}, RgCNN \cite{narayan2018learning}, PATCHY-SAN \cite{niepert2016learning}, DyGGNN \cite{taheri2019learning}, and others. Still, the approximation capability depends on the functions AGGREGATE and COMBINE designed for each GNN working on the single snapshot and the implemented Recurrent Neural Network. For example, the most general model, the original RNN, has been proven to be a universal approximator \cite{hammer2000approximation}.

As future work, extending all our results for graphs in continuous-time representation would be interesting. One difficulty in this context is deciding in which sense two continuous-time dynamic graphs are called WL equivalent since there are many possibilities for dealing with the given timestamps. 
The investigation of equivalences of dynamic graphs requires determining the handling of dynamic graphs that are equal in their structure but differ in their temporal occurrence, 
i.e., dependent on the commitment of the WL equivalence or the unfolding tree equivalence, it is required to decide whether the concepts need to be \textbf{time-invariant}. For time-invariant equivalence, the following concepts hold as they are. 
In case two graphs with the same structure should be distinguished when they appear at different times, the node and edge attributes can be extended by an additional dimension carrying the exact timestamp. 
Thereby, the unfolding trees of two (structural) equal nodes would be different, having different timestamps in their attributes. 
Then, all dynamic graphs $G^{(j)} \in \mathcal{G}$ are defined over the same time interval $I$. Without loss of generality, this assumption can be made since the set of timestamps of $G^{(j)}$ noted by $I_{G^{(j)}}$ can be padded by including missing timestamps $t_q$ and $G^{(j)}$ can be padded by empty graphs ${G^{(j)}_{q} = (\mathcal{V}^{(j)}_q = \emptyset, \mathcal{E}^{(j)}_q = \emptyset, \alpha_q(\emptyset) = \emptyset, \omega_q(\emptyset) = \emptyset}$. 

\medskip
Furthermore, this paper considers extensions of the usual 1-WL test and the commonly known unfolding trees. Further future work could be to investigate extensions, for example, the n-dim attributed/dynamic WL test or other versions of unfolding trees, covering GNN models not considered by the frameworks used in this paper. 
These extensions might result in a more exemplary classification of the expressive power of different GNN architectures. 

\medskip
Moreover, the proposed results mainly focus on the expressive power of GNNs. However, GNNs with the same expressive power may differ for other fundamental properties, e.g., the computational and memory complexity and the generalization capability. Understanding how the architecture of AGGREGATE$^{(i)}$, COMBINE$^{(i)}$, and READOUT impact those properties is of fundamental importance for practical applications of GNNs.

\section*{Acknowledgments}
This work was partially supported by the Ministry of Education and Research Germany (BMBF, grant number 01IS20047A) and partially by INdAM GNCS group.
We want to acknowledge Monica Bianchini and Maria Lucia Sampoli for their proofreading contribution and fruitful discussions.

\medskip

During the preparation of this work the author(s) used grammarly in order to improve the grammar of the work. After using this tool, the author(s) reviewed and edited the content as needed and take(s) full responsibility for the content of the publication.
\newpage

\appendix
\section{Appendix}\label{section_appendix}

\subsection{Proof of Prop.~\ref{f_unfold}}\label{apx:prop_functions_of_attr_unf_trees}
A function $f$ belongs to $\mathcal{F}(\mathcal{G}')$ if and only if there exists a function $\kappa$ defined on trees such that for any graph $G'\in\mathcal{G}'$ it holds $f(G',v)= \kappa(\mathbf{T}_v)$, for any node $v \in G'$.
\begin{proof} We prove by showing both equivalence directions:
\begin{compactenum}
    \item[$\Leftarrow$] If there exists a function $\kappa$ on attributed unfolding trees such that $f(G',v) =\kappa (\mathbf{T}_v)$ for all $v\in G'$, then $u\sim_{AUT}v$ for $u,v\in G$ implies
        $f(G',u) =\kappa (\mathbf{T}_u) = \kappa (\mathbf{T}_v) = f(G',v)$.
    \item[$\Rightarrow$] If $f$ preserves the attributed unfolding equivalence, then a function $\kappa$ on the attributed unfolding tree of an arbitrary node $v$ can be defined as $\kappa(\mathbf{T}_v):= f(G', v)$. Then, if $\mathbf{T}_u$ and $\mathbf{T}_v$ are two attributed unfolding trees, $\mathbf{T}_u=\mathbf{T}_v$ implies $f(G',u) = f(G',v)$ and $\kappa$ is uniquely defined. 
\end{compactenum}
\end{proof}

\subsection{Proof of Lem.~\ref{equal_trees=WL_coloring}}\label{proof_of_equal_trees=WL_coloring}

Consider $G' = (\mathcal{V}', \mathcal{E}' , \alpha', \omega')$ as the SAUHG resulting from a transformation of an arbitrary static graph $G = (\mathcal{V}, \mathcal{E} , \alpha, \omega)$ with nodes $u, v \in \mathcal{V}$ and corresponding attributes $\alpha_u, \alpha_v$. Then it holds 
\begin{align}\label{eq4}
    \forall \; d \in \mathbb{N}_0 : \mathbf{T}^{d}_u = \mathbf{T}^{d}_v \Longleftrightarrow c_u^{(d)} = c_v^{(d)}
\end{align}

\begin{proof}
The proof is carried out by induction on $d$, which represents both the
depth of the unfolding trees and the iteration step in the WL coloring.
\begin{compactenum}
    \item[\underline{$d = 0$:}] It holds 
    \begin{equation*}
         \quad \mathbf{T}^0_u = Tree(\alpha'_u) = Tree(\alpha'_v) = \mathbf{T}^0_v \;
    \end{equation*}
    \begin{equation*}
        \Longleftrightarrow \; \alpha'_u = \alpha'_v \text{ and } c^{(0)}_u = \text{HASH}(\alpha'_u) = \text{HASH}(\alpha'_v) = c^{(0)}_v.
    \end{equation*}
    \item[\underline{$d > 0$:}] Suppose that Eq.~\eqref{eq4} holds for $d-1$, and prove that it holds also for $d$.
\begin{compactenum}
    \item[-] By definition,  $\mathbf{T}^d_u = \mathbf{T}^d_v$ is equivalent to
        \begin{equation}
            \begin{cases}
            \mathbf{T}^{d-1}_u = \mathbf{T}^{d-1}_v \quad 
               \text{and} \\
           \text{\small Tree}\bigl(\alpha'_u, \Omega'_{ ne_u}, \mathbf{T}^{d-1}_{ne_u}\bigr)  = \text{\small Tree}\bigl(\alpha'_v, \Omega'_{ ne_v}, \mathbf{T}^{d-1}_{ne_v}\bigr).
           \end{cases}\label{eq6}
           \end{equation}
    \item[-] Applying the induction hypothesis, it holds that
        \begin{equation}\label{eq7}
            \mathbf{T}^{d-1}_u = \mathbf{T}^{d-1}_v \Longleftrightarrow c^{(d-1)}_u = c^{(d-1)}_v.
        \end{equation}
    \item[-] Eq.~\eqref{eq6} is equivalent to the following: 
        \begin{equation*}
            \alpha'_u = \alpha'_v, \quad \Omega'_{ ne_u} = \Omega'_{ ne_v}  \quad \text{and} \quad \mathbf{T}^{d-1}_{ne_u} = \mathbf{T}^{d-1}_{ne_v}.
        \end{equation*}
        Given the definition of the unfolding trees and their construction, this is equivalent to
        \begin{equation}
        \begin{cases}
             \omega'_{\{u, {u_{i}}\}} = \omega'_{\{v, {v_{i}}\}} & \forall \; u_i\in ne_{u},\ v_i\in ne_{v}\; \text{ and } \\ 
             \mathbf{T}^{d-1}_{{u_{i}}} = \mathbf{T}^{d-1}_{{v_{i}}} & \forall\; u_i\in ne_{u},\ v_i\in ne_{v}. \label{eq89}
             \end{cases}
        \end{equation}
    \item[-] 
    By the induction hypothesis, 
Eq.~\eqref{eq89} is equivalent to 
        \begin{equation*}
            c^{d-1}_{ne_{u}} = c^{d-1}_{ne_{v}} 
            \text{, \; i.e.,}
            \end{equation*}
            \begin{equation*} \quad
            \sms{c^{(d-1)}_{u_i} \mid u_i \in ne_u} \; = \; \sms{c^{(d-1)}_{v_i} \mid v_i \in ne_v}. 
        \end{equation*}
    \item[-] Putting together Eq.~\eqref{eq7}, \eqref{eq89}, and the fact that the HASH function is bijective, we obtain:
        \begin{equation*}
        \begin{split}
            &\text{HASH}\Bigl(\bigl(c^{(d-1)}_u, \Omega'_{ ne_u}, \sms{c^{(d-1)}_{u_i} \mid u_i \in ne_u}\bigr)\Bigr) \\
            &=\text{HASH}\Bigl(\bigl(c^{(d-1)}_v, \Omega'_{ ne_v}, \sms{c^{(d-1)}_{v_i} \mid v_i \in ne_v}\bigr)\Bigr)
        \end{split}
        \end{equation*}
        which, by definition, is equivalent to $c^{(d)}_u = c^{(d)}_v$.
\end{compactenum}
\end{compactenum}
\end{proof}

\subsection{Proof of Prop.~\ref{f_unfold_dyn}}\label{proof_of_f_unfold}
A dynamic system \straight{dyn} belongs to $\mathcal{F}(\mathcal{D})$ if and only if there exists a function $\kappa$ defined on attributed trees such that for all ${(t,G,v)\in \mathcal{D}}$ it holds 
\begin{align*}
\text{\straight{dyn}}(t,G,v)= \kappa\Bigl(\bigl(\mathbf{T}_{v}(i)\bigr)_{i \in [t]}\Bigr).
\end{align*}

\begin{proof} We show the proposition by proving both directions of the equivalence relation:
\begin{compactenum}
    \item[$\Rightarrow$:] If there exists $\kappa$ such that $\text{\straight{dyn}}(t,G,v) = \kappa\Bigl(\bigl(\mathbf{T}_{v}(i)\bigr)_{i\in [t]}\Bigr)$ for all triplets $(t,G,v)\in \mathcal{D}$, then for any pair of nodes $u \in G_1, v \in G_2$ with $u \sim_{DUT} v$ it holds
    {\small\begin{equation*}
        \text{\straight{dyn}}(t,G_1,u) = \kappa\Bigl(\bigl(\mathbf{T}_{u}(i)\bigr)_{i\in [t]}\Bigr) = \kappa\Bigl(\bigl(\mathbf{T}_{v}(i)\bigr)_{i\in [t]}\Bigr) = \text{\straight{dyn}}(t,G_2,v).
    \end{equation*}}
    \item[$\Leftarrow$:] On the other hand, if \text{\straight{dyn}} preserves the unfolding equivalence, then we can define $\kappa$ as
    \begin{equation*}
        \kappa\Bigl(\bigl(\mathbf{T}_{v}(i)\bigr)_{i\in [t]}\Bigr) = \text{\straight{dyn}}(t,G,v).
    \end{equation*}
    Note that the above equality is a correct specification for a function. In fact, if
    \begin{equation*}
        \kappa\Bigl(\bigl(\mathbf{T}_{v}(i)\bigr)_{i\in [t]}\Bigr) = \kappa\Bigl(\bigl(\mathbf{T}_{u}(i)\bigr)_{i\in [t]}\Bigr) 
    \end{equation*}
    implies $\text{\straight{dyn}}(t,G,u) = \text{\straight{dyn}}(t,G,v)$, then $\kappa$ is uniquely defined.
\end{compactenum}
\end{proof}

\subsection{Sketch of the proof of Thm.~\ref{theorem_universal_approx_sauhg}}\label{apx:proof_apx_thm_attributed}
Let $\mathcal{G}'$ be the domain of bounded SAUHGs  with the maximal number of nodes ${N= \max\limits_{G' \in \mathcal{G}'} |G'|}$.
For any measurable function $f \in \mathcal{F}(\mathcal{G}')$ preserving 
the attributed unfolding equivalence (cf.~Def.~\ref{def_preserve_attr_unf_eq}), any norm $\| \cdot \|$ on $\mathbb{R}$, any probability measure $P$ 
on $\mathcal{G}'$, for any reals $\epsilon, \lambda$ where $\epsilon, \lambda >0$, there exists a SGNN defined by the continuously differentiable functions $\text{COMBINE}^{(i)}$, $\text{AGGREGATE}^{(i)}$, at iteration $ i \leq 2N-1$, and by the function
$\text{READOUT}$, with hidden dimension $r=1$, i.e,  $h_v^i\in\mathbb{R}\ \forall i$, such that the function
$ \varphi$  (realized by the SGNN) computed after $2N-1$ steps for all $G'\in\mathcal{G}'$
satisfies the condition
\begin{equation*}
P( \| f(G',v)- \varphi( G',v) \| \leq \varepsilon) \geq 1- \lambda.
\end{equation*}

\begin{proof}
Since the proof proceeds analoguously to the one in \cite{d2021unifying}, we will only sketch the proof idea here and refer to the original paper for further details. First, we need a preliminary lemma which is an extension of \cite[Lem.~1]{scarselli2008computational} to the domain of SAUHGs $\mathcal{G}'$. Intuitively, this lemma suggests that a domain of SAUHG graphs with continuous attributes can be partitioned into 
small subsets so that the attributes of the graphs are almost constant in each partition. Moreover, in probability, a finite number of partitions is sufficient to cover a large part of the domain.
\begin{Lem}
\label{hypercubes_att}

For any probability measure $P$ on $\mathcal{D}'$, and any reals $\lambda$, $\delta$, where $\lambda > 0$, $\delta \geq 0$, there exists a real $\Bar{b} >0$, which is independent of $\delta$, a set $\Bar{\mathcal{D}'} \subseteq \mathcal{D}'$, and a finite number of partitions $\Bar{\mathcal{D}'_1}, \dots , \Bar{\mathcal{D}'_p}$ of $\Bar{\mathcal{D}'}$, where $\Bar{\mathcal{D}'_j} = \mathcal{G}'_j \times \{ v_j \}$, with $\mathcal{G}'_j \subseteq \mathcal{G}'$ and $v_j \in \mathcal{G}'_j$, such that: 
\begin{compactenum}
    \item $P(\Bar{\mathcal{D}'}) \geq 1- \lambda $ holds;
    \item for each $j$, all the graphs in $\mathcal{G}'_j$ have the same structure, i.e., they differ only in the values of their attributes;
    \item for each set $\Bar{\mathcal{D}'_j}$, there exists a hypercube $\mathcal{H}_j\subset\mathbb{R}^{NM2k}$ such that $\gamma_{G} \in \mathcal{H}_j$ holds for any graph $G' \in \mathcal{G}'_j$ with $N=max_{G'\in\mathcal{G}'}|\mathcal{V}'|$ and $M=max_{G'\in\mathcal{G}'}|\mathcal{E}'|$. Here, $\gamma_{G'}$ denotes the vector obtained by concatenating all the attribute vectors of both nodes and edges of $G'$, namely $\gamma_{G'} = [A_{G'} | \Omega_{G'}]$, where $A_{G'}$ is the concatenation of all the node attributes and $\Omega_{G'}$ is the concatenation of all edge attributes;
    \item for any two different sets $\mathcal{G}'_i$, $\mathcal{G}'_j$, $i \neq j$, their graphs have different structures, or their hypercubes $\mathcal{H}_i$, $\mathcal{H}_j$ are disjoint, i.e., $\mathcal{H}_i \bigcap \mathcal{H}_j = \emptyset$;
    \item for each $j$ and each pair of graphs $G_1$, $G_2 \in \mathcal{G}'_j$, the inequality $\| \gamma_{G_1} - \gamma_{G_2} \|_{\infty} \leq \delta$ holds;
    \item for each graph $G' \in \Bar{\mathcal{D}'}$, the inequality $\| \gamma_{G'}\|_{\infty} \leq \Bar{b}$ holds.
\end{compactenum}
\end{Lem}

\begin{proof}
The proof is similar to the one contained in \cite{scarselli2008computational}. The only remark needed here is that we can consider the whole concatenating of all attributes from both nodes and edges without loss of generality; indeed, if we were considering the node and the edge attributes separately, we would need conditions on the hypercubes, s.t.:
\begin{equation*}
\begin{split}
   & \quad  \| A_{G_1} - A_{G_2} \|_{\infty} \leq \delta^A, \; \delta^A>0, \\ \text{ and }  &\quad 
     \| \Omega_{G_1} - \Omega_{G_2} \|_{\infty} \leq \delta^\Omega, \; \delta^\Omega>0.
\end{split}
\end{equation*}
Then we can stack those attribute vectors, as in the statement, s.t. :
\begin{equation*}
\hspace{-.5cm}
\begin{split}
    \| \gamma_{G_1} - \gamma_{G_2} \|_{\infty} &= \| \bigl([A_{G_1} | \mathbb{0} ] + [ \mathbb{0} | \Omega_{G_1}]\bigr) - \bigl([A_{G_2} | \mathbb{0} ] + [ \mathbb{0} | \Omega_{G_2}] \bigr) \|_{\infty}
    \\
    & \leq \| A_{G_1} - A_{G_2} \|_{\infty} + \| \Omega_{G_1} -  \Omega_{G_2} \|_{\infty} \\ & \leq \delta^A + \delta^\Omega := \delta
\end{split}
\end{equation*}
which allows us to exploit the same proof contained in \cite{scarselli2008computational}.
\end{proof}

The following theorem, where the domain contains a finite number of graphs and the attributes are integers, is equivalent to Thm.~\ref{theorem_universal_approx_sauhg}.



\begin{Thm} \label{attributed reduct}
For any finite set of $p$ patterns\\ ${\{ ( G'_j , v) |\  G'_j \in \mathcal{G}', v \in \mathcal{V}'_j, j\in [p] \}}$, with the maximal number of nodes in the domain $N= \max_{G' \in \mathcal{G}'} |G'|$, for any function which preserves the attributed unfolding equivalence, and for any real $\varepsilon >0$, there exist continuously differentiable functions  $\text{AGGREGATE}^{(i)}$, $\text{COMBINE}^{(i)}$, $\forall \; i \leq 2N-1$, s.t.
{\small\begin{equation*}
\hspace*{-0.1cm}
    \mathbf{h}_v^i = \text{COMBINE}^{(i)}\left( \mathbf{h}_v^{(i-1)}, \text{AGGREGATE}^{(i)}\left(\{ \mathbf{h}_u^{i-1}\}_{ u \in ne_v}, \{\omega_{\{u,v\}}\}_{ u \in ne_v}  \right)\right)
\end{equation*}}

and a function
$\text{READOUT}$, with hidden dimension $r=1$, i.e, $\mathbf{h}_v^i\in\mathbb{R}$, so that the function
$ \varphi$  (realized by the SGNN), computed after $N$ steps,
satisfies the condition
\begin{equation}
|\tau (G'_j, v) - \varphi(G'_j, v)| \leq \varepsilon \quad \text{ for any } v\in \mathcal{V}'_j.
\end{equation}
\end{Thm}

\begin{proof}[Sketch of the proof]
The idea of the proof is designing a GNN that can approximate any function $\tau$ that preserves the attributed unfolding equivalence. According to Thm.~\ref{f_unfold} there exists a function $\kappa$, s.t.
\begin{equation*}
    \tau(G'_j,v) = \kappa (T_v).
\end{equation*}
Therefore, the GNN has to encode the attributed unfolding tree into the node attributes, i.e., for each node $v$, we want to have $\mathbf{h}_v = \triangledown (\mathbf{T}_v)$, where $\triangledown$ is an encoding function that maps attributed unfolding trees into real numbers. The existence and injectiveness of $\triangledown$ are ensured by construction. More precisely, the encodings are constructed recursively by the $\text{AGGREGATE}^{(i)}$ and the $\text{COMBINE}^{(i)}$
functions using the neighborhood information, i.e., the node and edge attributes.

Consequently, the theorem can be proven given that there exist
appropriate functions $\triangledown$, $\text{AGGREGATE}^{(i)}$, $\text{COMBINE}^{(i)}$ and READOUT.
For this purpose, the functions  $\text{AGGREGATE}^{(i)}$ and $\text{COMBINE}^{(i)}$ must satisfy  $\forall \; i \leq 2N-1$:
\begin{equation*}
\hspace{-0.5cm}
\begin{array}{lc}
\triangledown(\mathbf{T}_v^i)= \mathbf{h}_v^i & \\ 
= \text{\scriptsize COMBINE}^{(i)}\left( \mathbf{h}_v^{(i-1)}, \text{\scriptsize AGGREGATE}^{(i)}\left(\{ \mathbf{h}_u^{i-1}\}_{ u \in ne_v}, \{\omega'_{\{u,v\}}\}_{ u \in ne_v}  \right)\right)  & \\
=  \text{\scriptsize COMBINE}^{(i)} \left(\triangledown(\mathbf{T}_v^{i-1}), \text{\scriptsize AGGREGATE}^{(i)}\left(\{\triangledown(\mathbf{T}^{i-1}_{u})\}_{u \in ne_v}, \Omega'_{ne_v}\right)\right). & \\
\end{array}
\end{equation*}

In a simple solution, $\text{AGGREGATE}^{(i)}$ decodes the attributed trees of the neighbors $u$ of $v$, $\mathbf{T}^{i-1}_{u}$, and stores them into a data structure
to be accessed by $\text{COMBINE}^{(i)}$. The detailed construction of the appropriate functions is given in \cite{d2021unifying}.
\end{proof}
Adopting an argument similar to that in \cite{scarselli2008computational}, it is proven that the previous theorem is equivalent to Thm.~\ref{theorem_universal_approx_sauhg} and this concludes the proof. 
\end{proof}

\subsection{Proof of Thm.~\ref{dyn_thm_approximation}}\label{proof_approx}
Let  $G = (G_t)_{t\in I}$ be a discrete dynamic graph in the graph domain $\mathcal{G}$ and ${N= \max_{G \in \mathcal{G}} |G|} $ be the maximal number of nodes in the domain.
Let $\text{\straight{dyn}}(t,G,v) \in \mathcal{F}(\mathcal{D})$ be any measurable dynamical system preserving 
the unfolding equivalence,  $\| \cdot \|$ be a norm  on $\mathbb{R}$, $P$ be any probability measure 
on $\mathcal{D}$ and $\epsilon, \lambda$ be any real numbers where $\epsilon,\lambda >0$. Then, there exists a DGNN composed by SGNNs with $2N-1$ layers and hidden dimension $r=1$, and Recurrent Neural Network state dimension
$s=1$ such that the function $ \varphi$ realized by this model satisfies 
\begin{equation*}
P( \| \text{\straight{dyn}}(t,G,v)- \varphi(t, G,v) \| \leq \varepsilon) \geq 1- \lambda \quad \quad \quad \forall t \in I.
\end{equation*} 

\begin{proof}
To prove the theorem above, we need some preliminary results.
Using the same argument used for SAUHGs in \ref{theorem_universal_approx_sauhg}, we  need, as a preliminary lemma, the extension of \cite[Lem.~1]{scarselli2008computational} to the domain of dynamic graphs $\mathcal{D}$, analogously to the extension to the domain of SAUHGs in \ref{hypercubes_att}.

\begin{Lem} \label{hypercubes}
Lemma \ref{hypercubes_att} holds for the domain of dynamic graphs $\mathcal{D}$.
\end{Lem}
\begin{proof}
Indeed, taking into account the argument in \cite{thomas2021graph}, one can establish a bijection between the domain of dynamic graphs and the domain of SAUHGs; on the latter, we can directly apply \ref{hypercubes_att}.
\end{proof}

Thm.~\ref{dyn_thm_approximation} is equivalent to the following, where the domain contains a finite number of elements in $\mathcal{D}$ and the attributes are integers.

\begin{Thm}\label{dyn_reduct}
For any finite set of $p$ patterns $$\{ (t^{(j)}, G^{(j)} , v^{(j)}) |\ (t^{(j)}, G^{(j)} , v^{(j)}) \in \mathcal{D}, j\in [p] \}$$ with  the maximal number of nodes ${N= \max_{G \in \mathcal{G}} |G|}$ 
and with graphs having integer features, for any measurable dynamical system preserving the unfolding equivalence,  $\| \cdot \|$ be a norm  on $\mathbb{R}$, $P$ be any probability measure 
on $\mathcal{D}$ and $\epsilon$ be any real number where $\epsilon >0$. 
Then, there exists a DGNN as defined in Def.~\ref{DGNN} such that the function
$ \varphi$  (realized by this model) computed after $N$ steps satisfies the condition

\begin{equation}\label{mainIntEq}
||dyn (t^{(j)}, G^{(j)}, v^{(j)}) - \varphi(t^{(j)}, G^{(j)}, v^{(j)})|| \leq \varepsilon \quad 
\end{equation}
$\forall \; j\in [p] \; \text{ where} \;  t^{(j)}\in I.$
\end{Thm}

The equivalence between Thm. ~\ref{dyn_thm_approximation} and Thm. ~\ref{dyn_reduct}  is formally proven by the following lemma.

\begin{Lem}\label{reducLemma}
Thm.~\ref{dyn_thm_approximation} holds if and only if Thm.~\ref{dyn_reduct} holds.
\end{Lem}

\begin{proof}
    The proof is similar to the one contained in ~\cite{scarselli2008computational}. Nevertheless, we want to highlight that in this case, patterns are taken from $\mathcal{D}:= I \times \mathcal{G} \times  \mathcal{V} $, as we are proving it in the context of the dynamic graphs.
\end{proof}

\noindent
Now, we can  proceed to prove Thm.~\ref{dyn_reduct}.

\begin{proof}[Proof of Thm. \ref{dyn_reduct}]
The proof of this theorem involves assuming that the output dimension is $m=1$, i.e., $\text{dyn}(t,G,v) \in \mathbb{R}$, but the result can be extended to the general case with $m\in\mathbb{N}$ by concatenating the corresponding results. As a result of Thm.~\ref{f_unfold_dyn}, there exists a function $\kappa$, s.t. $\text{dyn}(t,G,v) = g(x_v(t)) = \kappa ((\mathbf{T}_v(i))_{i\in [t]})$ where $\mathbf{T}_v(i)$ is an attributed unfolding tree. Given $N_t$ as the number of nodes of the graph at timestep $t$, in order to store the graph information, an attributed unfolding tree of depth $2N_t-1$ is required for each node, in such a way that $\kappa$ can satisfy
\begin{equation*}
    dyn(t,G,v) = \kappa ((\mathbf{T}_v(i))_{i \in [t]})= \kappa ((\mathbf{T}_v^{N_t}(i))_{i\in [t]}).
\end{equation*}

The required depth is a straight consequence of Theorem 4.1.3 in \cite{d2021unifying}, which we briefly report here.
\begin{Thm}{\cite[Theorem 4.1.3]{d2021unifying}}
    The following statements hold for graphs with at most $N$ nodes.
\begin{enumerate}
\item
  Let $\mathbf{G}$ and $\mathbf{H}$ be connected graphs and $x,y$ be nodes of  $\mathbf{G}$ and $\mathbf{H}$, respectively. The infinite unfolding trees  $\mathbf{T}_x, \mathbf{T}_y$ are equal if and only if
    they are equal up to depth $2N-1$, i.e., $\mathbf{T}_x = \mathbf{T}_y$ iff $\;\mathbf{T}_x^{2N-1} = \mathbf{T}_y^{2N-1}$.
\item For any $N$, there exist two graphs $\mathbf{G}$ and $\mathbf{H}$ with nodes $x,y$, respectively,
    such that the infinite unfolding trees  $\mathbf{T}_x, \mathbf{T}_y$ are different, but they are equal up to
    depth $2N-16 \sqrt{N}$, i.e.,   $\mathbf{T}_x \neq \mathbf{T}_y$ and $\mathbf{T}_x^t = \mathbf{T}_y^t$ for $i \leq 2N-16 \sqrt{N}$. \hfill \qed
\end{enumerate}
\label{th:treeDepth}
\end{Thm}
The Theorem is tightly connected to the results previously displayed in \cite{krebs2015universal}.
\\

The main idea behind the proof of Theorem \ref{dyn_reduct} is to design a DGNN that can encode the sequence of attributed unfolding trees $(\mathbf{T}_v(i))_{i\in [t]}$ into the node attributes at each timestep t, i.e, $\mathbf{q}_v(t) = \#_t ((\mathbf{T}_v(i))_{i\in [t]}) $. This is achieved by using a coding function that maps sequences of $t+1$ attributed trees into real numbers. To implement the encoding that could fit the definition of the DGNN, two coding functions are needed: the $\nabla$ function, which encodes the attributed unfolding trees, and the family of coding functions $\#_t$. 
The composition of these functions is used to define the node's attributes, and the DGNN can produce the desired output by using this encoded information as follows:

\begin{equation}
    \begin{split}
        \mathbf{q}_v(0) & = \mathbf{h}_v(0) =\#_0\left(\nabla ^{-1} (\mathbf{h}_v(0))\right) \\
        \mathbf{q}_v(t) & = \#_t \left(\text{APPEND}_t\left(\#_{t-1}^{-1}(\mathbf{q}_v(t-1)), \nabla ^{-1}(\mathbf{h}_v(t))\right)\right)
    \end{split}
\end{equation}

where the ausiliar function $\text{APPEND}_t$ and the $\nabla$, $\#_t$ coding functions are defined in the following.  \\

\noindent \textbf{APPEND$_t$} \\Let $\mathcal{T}^{d}(v)$ be the domain of the attributed unfolding trees with root $v$, up to a certain depth $d$.
The function\\ {$\text{APPEND}_t : \{(\mathbf{T}_{v}^d(i))_{i\in[t-1]}\}\cup \emptyset \times \mathcal{T}^{d}(v)   \rightarrow \{(\mathbf{T}_{v}^d(i))_{i\in[t]}\}$} is defined as follows: 
\begin{equation*}
\begin{split}
        &\text{APPEND}_0\bigl(\emptyset, \mathbf{T}_{v}^d(0)\bigr) := \mathbf{T}_{v}^d(0)\\
       &\text{APPEND}_t\bigl( \bigl({\mathbf{T}_{v}^d(0), \ldots, \mathbf{T}_{v}^d(t-1)}\bigr),\mathbf{T}_{v}^d(t)\bigr) \\
        &:=  \bigl(\mathbf{T}_{v}^d(0), \ldots,\mathbf{T}_{v}^d(t-1), \mathbf{T}_{v}^d(t)\bigr)
\end{split}
\end{equation*}
\noindent 
Intuitively, this function appends the unfolding tree snapshot of the node $v$ at time $t$ to the sequence of the unfolding trees of that node at the previous $t-1$ timesteps. 


\noindent In the following, the coding functions are defined; their existence and injectiveness are provided by construction. \\

\noindent\textbf{The $\nabla$ Coding Function}\\
Let  $\nabla:= \mu_{\nabla} \circ \nu_{\nabla}$ be a composition of any two  injective functions $\mu_{\nabla}$ and $\nu_{\nabla}$ with the following properties: 
\begin{compactenum}
\item[-] $\mu_{\nabla}$ is an injective function from the domain of static unfolding trees, calculated 
on the nodes in the graph $G_t$,  to the Cartesian product $\mathbb{N} \times \mathbb{N}^P \times \mathbb{Z}^{A}= \mathbb{N}^{P+1} \times \mathbb{Z}^{A} $,  where
 $P$ is the maximum number of nodes a tree could have. 
 
Intuitively, in the Cartesian product, $\mathbb{N}$ represents the tree structure, $\mathbb{N}^P$ denotes the node numbering, while, for each node, an integer vector in $\mathbb{Z}^{A}$ is used to encode the node attributes. Notice that  $\mu_{\nabla}$ exists and is injective since the maximal information contained in an unfolding tree
is given by the union of all its node attributes and all its structural information, which just equals the dimension of the codomain of $\mu_{\nabla}$.
\item[-] $\nu_{\nabla}$ is an injective function from $\mathbb{N}^{P+1} \times \mathbb{Z}^{A}$ to $\mathbb{R}$, whose existence is guaranteed by  the cardinality theory, since the two sets have the same cardinality. 
\end{compactenum}
Since $\mu_{\nabla_t}$ and $\nu_{\nabla_t}$ are injective, also the existence and the injectiveness of
 $\nabla_t$ is ensured.\\

\noindent\textbf{The $\#_t$ Coding Family } \label{hash_def}

\noindent Similarly to $\nabla$, the functions $\#_t :=\mu_{\#_t} \circ \nu_{\#_t}$  are composed by two functions $\mu_{\#_t}$ and $\nu_{\#_t}$ with the following properties: 
\begin{compactenum}
\item[-] $\mu_{\#_t}$ is an injective function from the domain of the dynamic unfolding trees $ \mathcal{T}^d_t(v) := \{(\mathbf{T}^{d}_{v}(i))_{i\in [t]}\}$
 to the Cartesian product $\mathbb{N}^{t} \times \mathbb{N}^{tP_t} \times \mathbb{Z}^{tA}= \mathbb{N}^{t(P_t+1)} \times \mathbb{Z}^{t A} $,  where
$P_t$ is the maximum number of nodes a tree could have at time t.
\item[-] $\nu_{\#_t}$ is an injective function from $\mathbb{N}^{t(P+1)} \times \mathbb{Z}^{t A}$ to $\mathbb{R}$, whose existence is guaranteed by  the cardinality theory, since the two sets have the same cardinality. 
\end{compactenum}
Since $\mu_{\#_t}$ and $\nu_{\#_t}$ are injective, also the existence and the injectiveness of
 $\#_t$ are ensured. \\

 \noindent \textbf{The recursive function $\text{f}$, $\text{AGGREGATE}_t^{(i)},$ $\text{COMBINE}_t^{(i)}$}
 
\noindent The recursive function $f$ has to satisfy
\begin{equation*}
f\bigl(\mathbf{q}_v(t-1), \mathbf{h}_v(t)\bigr) = \#_t\bigl((\mathbf{T}_v(i))_{i\in[t]}\bigr)= \mathbf{q}_v(t) ,
\end{equation*}
where the $\mathbf{h}_v(t)$ is the hidden representation of node $v$ at time $t$ extracted from the $t$-th SGNN, i.e., $\mathbf{h}_v(t) =$ $\text{SGNN}$ $(G_t,v)$. In particular, at each iteration $i$, we have
{\footnotesize
\begin{equation*}
\hspace*{-.4cm} 
    \mathbf{h}^i_v(t)=\text{COMBINE}_t^{(i)} \left(\mathbf{h}_v^{i-1}(t), \text{AGGREGATE}^{(i)}\left(\{ \mathbf{h}_u^{i-1}(t)\}_{ u \in ne_v(t)}, \{\omega_{\{u,v\}}(t)\}_{ u \in ne_v(t)}  \right)\right)
\end{equation*}}
Further, the functions $\text{AGGREGATE}_t^{(i)}$ and $\text{COMBINE}_t^{(i)}$ -- following the proof in \cite{d2021unifying} -- must satisfy 

\begin{equation*}
\hspace*{-1.35cm} 
\begin{array}{lc}
\triangledown(\mathbf{T}_v^i(t))= \mathbf{h}_v^i(t)= & \\\text{\scriptsize COMBINE}_t^{(i)} \left(\mathbf{h}_v^{i-1}(t), \text{\scriptsize AGGREGATE}_t^{(i)}\left(\{\mathbf{h}^{i-1}_{u}(t) \; \}_{u \in ne_v(t)}, \{\omega_{\{u,v\}}(t)\}_{ u \in ne_v(t)}  \right)\right)&  \\
= \text{\scriptsize COMBINE}_t^{(i)} \left(\triangledown(\mathbf{T}_v^{i-1}(t)), \text{\scriptsize AGGREGATE}_t^{(i)}(\{\triangledown(\mathbf{T}^{i-1}_{u}(t)) \; \}_{u \in ne_v(t)})\right) &
\end{array}
\end{equation*}
$\forall \; i \leq 2N-1$ 
and $\forall \; t\in I$.

For example, the trees can be collected into the coding of a new  tree, i.e.,
{\small\begin{equation*}\hspace*{-0.4cm}
    \text{AGGREGATE}_t^{(i)}({\triangledown}(\mathbf{T}^{i-1}_{u}(t) ), {u \in ne_v(t)})= {\triangledown}(\cup_{u \in ne_v(t)} {\triangledown}^{-1}(\triangledown(\mathbf{T}^{i-1}_{u}(t)))),
\end{equation*}}
where  $\cup_{u \in ne_v(t)} $ denotes an operator that  constructs a tree with a  root having void attributes from a set of subtrees (see Fig.~\ref{fig:union}). Then, $\text{COMBINE}_t^{(i)}$ assigns the correct attributes to the root by extracting them from  $\mathbf{T}^{i-1}_{v}(t) $, i.e.,
{\small\begin{equation*}
\hspace*{-0.2cm}
    \text{COMBINE}_t^{(i)}({\triangledown}(\mathbf{T}^{i-1}_{v}(t)),b)= {\triangledown}( \text{ATTACH}({\triangledown}^{-1}({\triangledown}(\mathbf{T}^{i-1}_{v}(t))) ,{\triangledown}^{-1}(b))),
\end{equation*}}
where ATTACH is an operator that  returns a tree constructed by replacing the attributes of the root in the latter tree with those of the former tree and $b$ is the result of the $\text{AGGREGATE}_t^{(i)}$ function. 
\begin{figure}[ht]
\centering
 \includegraphics[width=\linewidth]{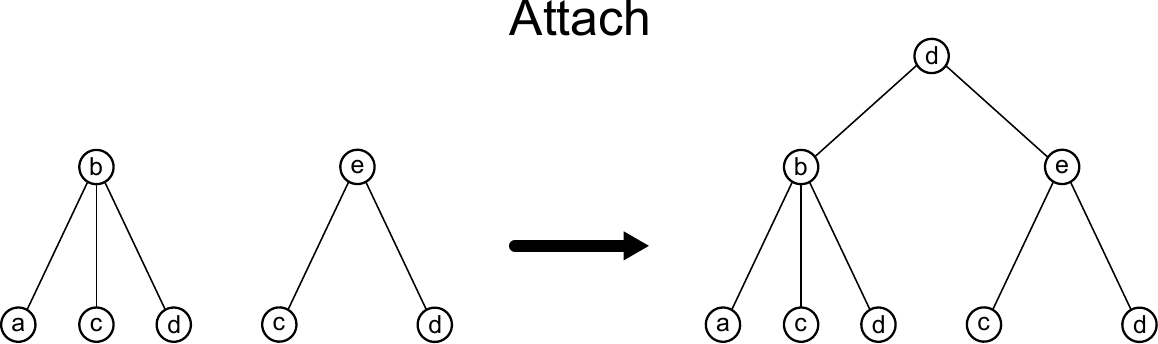}
\caption{The ATTACH operator on trees.}
\label{fig:union}
\end{figure}
\\Now, notice that, with this definition, $\text{AGGREGATE}_t^{(i)}$, $\text{COMBINE}_t^{(i)}$, and $\text{READOUT}_{dyn}$ may not be differentiable. Nevertheless, Eq.~(\ref{mainIntEq}) has to be satisfied only for a finite number of graphs,
namely $G_j$. 
Thus, we can specify other functions $\overline{\text{AGGREGATE}_t}^{(i)}$, $\overline{\text{COMBINE}_t}^{(i)}$, and $\overline{\text{READOUT}}$, which  produce exactly the same computations when they are applied on the graphs 
$G_j$, but that can be extended to the rest of their domain so that they are continuously differentiable. Obviously, such an extension exists since those  functions are only constrained to interpolate a finite number of points \footnote{Notice
that a similar extension can  also be applied to the coding function $\triangledown$ and to the
decoding function $\triangledown^{-1}$. In this case, the coding function is not injective on the whole domain, but only
on the graphs mentioned in the theorem.}. \\
\\
 \noindent \textbf{The $\text{READOUT}_{\text{dyn}}$ function}

\noindent Eventually,  $\text{READOUT}_{\text{dyn}}$  must satisfy:

\begin{equation*}
    \kappa (\cdot) := \text{READOUT}_{\text{dyn}}(\#_t (\cdot ))
\end{equation*}

so that, ultimately, 

\hspace{-2.5cm}
\begin{align*}
     &\text{dyn}(t,G,v) =\\
    & \text{\scriptsize READOUT}_{\text{\scriptsize  dyn}}\left(\#_t \left(\text{\scriptsize  APPEND}_t\left(\#_{t-1}^{-1}(\mathbf{q}_v(t-1)), \nabla ^{-1}(\mathbf{h}_v(t))\right)\right)\right)
\end{align*}

\end{proof}
This concludes the proof via Lem.~\ref{reducLemma}. 
\end{proof}

\subsection{Proof of Thm.~\ref{dyn_mainNN}}\label{proof_dyn-mainNN}

Assume that the hypotheses of Thm.~\ref{dyn_thm_approximation} are fulfilled and 
	$\cal{Q}_D$ is a class of discrete DGNNs with universal components.
	Then, there exists a parameter set $\theta$, and the functions 
	$\text{SGNN}(t)_\theta$, $f_{0,\theta}$, $f_\theta$,  implemented by Neural Networks in $\cal{Q}_D$, such that the thesis of	Thm.~\ref{dyn_thm_approximation} holds. 

\begin{proof}
The idea of the proof follows from the same reasoning adopted in \cite{d2021unifying}. 
Intuitively,  since the discrete DGNN of Thm.~\ref{dyn_thm_approximation} is implemented by continuously differentiable functions, its output depends continuously on the possible changes in the DGNN implementation:
small changes in the function implementation cause small changes in the DGNN outputs.
Therefore, the functions of the DGNN of Thm.~\ref{dyn_thm_approximation}  can be replaced by Neural Networks, provided that those networks are suitable approximators.

As in the proof of the dynamic version of the approximation theorem, cf.~Thm.~\ref{dyn_thm_approximation}, without loss of generality, we will assume that the attribute dimension is $n=1$\footnote{A GNN can theoretically be modeled with multiple components by stacking Neural Networks for each dimension, respectively.}. 

First of all, note that Thm.~\ref{dyn_thm_approximation} ensures that we can find continuously differentiable functions $\bar{f}$ , $\overline{\text{READOUT}_{dyn}}$ such that, for the corresponding function $\bar{\varphi}$ implemented by the DGNN it holds:
\begin{equation}\label{prob_partial}
P( \| \text{\straight{dyn}}(t,G,v)- \bar{\varphi}(t, G,v) \| \leq \frac{\varepsilon}{2}) \geq 1- \lambda \quad \forall \; t \in I,\, \epsilon,\lambda > 0.
\end{equation} 
Considering that the theorem has to hold only in probability, we can also assume that the domain is bounded to a finite set of patterns $\{ (t^{(i)},(G_t)_{t \in I}^{(i)}, v^{(i)}) \; | \; i=1, \dots, p  \}$ (as in Theorem \ref{dyn_reduct}). 
As a result, the functions , $\bar{f}$ and $\overline{\text{READOUT}}_{\text{dyn}}$ are bounded and have a bounded Jacobian. We can take the maximum of these Jacobians, which we will denote as $B$.

Moreover, let  $f_\theta$, $\text{READOUT}_{\text{dyn},\theta}$ be  universal components for DGNN, as in Def.~\ref{def:universal_dyn}, that approximate $\bar{f}$, $\overline{\text{READOUT}}_{\text{dyn}}$, respectively. Further, let  $\epsilon_1, \epsilon_2, > 0$ be the corresponding approximation errors, i.e.,  
\begin{equation}\label{def:universal_comp}
\begin{split}
\left\| \overline{f}( \mathbf{q}, \mathbf{h})-f_{\theta}( \mathbf{q},\mathbf{h})\right\|_\infty\leq \epsilon_1, \text{ and }\\  \left\| \overline{\text{READOUT}}_{\text{dyn}}(Q(t)) -\text{READOUT}_{\text{dyn},\theta}(Q(t)) \right\|_\infty\leq \epsilon_2
\end{split}
\end{equation}
hold $\forall \; t \in I$.

Now, from the proof of Theorem ~\ref{mainNN}  we know that 
\begin{equation*}
    P(\| \overline{\text{SGNN}}_i(G,v) - \text{SGNN}_{\theta,i}(G,v) \| \leq \epsilon_s) \geq 1 - \lambda_i
\end{equation*}

for $i\in [t]$, $\epsilon_s>0$ and for any norm. Then we can take every $\lambda_i$ small enough, s.t. 

\begin{equation*}
    \| \overline{\text{SGNN}}_i(G,v) - \text{SGNN}_{\theta,i}(G,v) \|_{\infty} \leq \epsilon_s 
\end{equation*}

holds on a finite set of patterns large enough to include those ones of the $i$-th timestep of each patterns of dynamic graphs on which Eq.~\eqref{prob_partial} holds.

Therefore, if we define  $\bar{\mathbf{h}}(t) : = \overline{\text{SGNN}}_i(G_t) $ and $\mathbf{h}_\theta(t) : = \text{SGNN}_{\theta,i}(G_t) $ we have

\begin{equation*}
    | \bar{\mathbf{h}}(t) - \mathbf{h}_\theta(t) \|_{\infty} =
\left\|\overline{\text{SGNN}}_i(G_t) - \text{SGNN}_{\theta,i}(G_t) \right\|_\infty\leq \epsilon_s.
\end{equation*}

In addition, let $\bar{H}(t)$ and $H_\theta(t)$ be the internal representations produced by $\overline{\text{SGNN}}$ and$\text{SGNN}_\theta$, stacked over all the nodes of the input graph. Then it holds
\begin{equation}\label{eq:bounded_diff}
\| \bar{H}(t) - H_\theta(t) \|_{\infty} \leq N\epsilon_s \; \forall \; t \in I, 
\end{equation}
where $ N = \text{max}_{G\in\mathcal{G}}|G|$ is the maximum number of nodes of the static graphs input in the bounded domain.
Let again  $\bar{Q}(0) := \bar{H}(0)$ and ${ \bar{Q}(t) := \bar{F}(\bar{Q}(t-1),\bar{H}(t))}$  be the stacking of the internal states produced by DGNN's internal recursive function $\bar{f}$ . Analoguously, let $Q_\theta(0) := H_\theta(0)$ and $ Q_\theta(t) := F_\theta(Q_\theta(t-1),H_\theta(t))$ be the output produced by the corresponding function of the parameterized DGNN.

Then it holds:

\begin{equation}\label{bound_zero}
  \| \bar{Q}(0) - Q_\theta(0) \|_{\infty}   =  \| \bar{H}(0) - H_\theta(0) \|_{\infty} \leq N\epsilon_s
\end{equation}

and

\begin{equation*}
\begin{split}
{\| \bar{f}(\bar{Q}(0), \cdot ) - \bar{f}( Q_{\theta}(0), \cdot) \|_{\infty} \leq B \|\bar{Q}(0) -Q_\theta(0) \|_{\infty}}\\
{\| \bar{f}(\cdot, \bar{H}(1) ) - \bar{f}( \cdot, H_\theta(1)) \|_{\infty} \leq B \|\bar{H}(1) -H_\theta(1) \|_{\infty}}
\end{split}
\end{equation*}
for a bound $B$ on the Jacobian of $\bar{f}(\mathbf{q},\mathbf{h})$ $\forall \; t \in I$ and $\forall \; \mathbf{q}$, which, along with Eq.~\eqref{eq:bounded_diff} and Eq.~\eqref{bound_zero} gives 

\begin{equation}\label{eq_bound_init_fct}
\begin{split}
\| \bar{f}(\bar{Q}(0), \cdot ) - \bar{f}( Q_{\theta}(0), \cdot) \|_{\infty} \leq N\epsilon_s B \\ 
\| \bar{f}(\cdot, \bar{H}(1) ) - \bar{f}( \cdot, H_\theta(1)) \|_{\infty} \leq N \epsilon_s B
\end{split}
\end{equation}

Therefore, we have that: \\
\underline{$t=1:$}
\begin{align*}
    &\| \bar{Q}(1) - Q_\theta(1) \|_{\infty} \\
    = & \| \bar{f}(\bar{Q}(0),\bar{H}(1)) - f_{\theta}(Q_{\theta} (0), H_\theta(1)) \|_{\infty} \\
    \overset{\text{add } 0}{=} &\|  \bar{f}(\bar{Q}(0),\bar{H}(1)) - \bar{f}(Q_\theta(0),\bar{H}(1))   \\
     &\quad +\bar{f}(Q_\theta(0),\bar{H}(1)) - \bar{f}(Q_\theta(0),H_\theta(1)) \\
    &\quad + \bar{f}(Q_\theta(0),H_\theta(1)) - f_\theta(Q_\theta(0),H_\theta(1)) \|_{\infty}\\
     \overset{\triangle\text{-ineq.}}{\leq}&  \| \bar{f}(\bar{Q}(0),\bar{H}(1)) - \bar{f}(Q_\theta(0),\bar{H}(1)) \|_{\infty} \\
    &\quad + \| \bar{f}(Q_\theta(0),\bar{H}(1)) - \bar{f}(Q_\theta(0),H_\theta(1)) \|_{\infty} \\
    &\quad + \| \bar{f}(Q_\theta(0),H_\theta(1)) - f_\theta(Q_\theta(0),H_\theta(1)) \|_{\infty} \\
     \overset{\eqref{eq_bound_init_fct}}{\leq}& 2 N \epsilon_s B  + N\epsilon_1 \\
     := &\lambda_1(\epsilon_s, \epsilon_1).
\end{align*}
\underline{$t> 0$:}
Analoguously, it follows for t>1 that
\begin{align*}
    &\| \bar{Q}(t) - Q_\theta(t) \|_{\infty} \\
    =&  \|  \bar{f}(\bar{Q}(t-1),\bar{H}(t)) - f_\theta(Q_\theta(t-1),H_\theta(t)) \|_{\infty}  \\
    =& \|  \bar{f}(\bar{Q}(t-1),\bar{H}(t)) - \bar{f}(Q_\theta(t-1),\bar{H}(t))   \\
     &\quad +\bar{f}(Q_\theta(t-1),\bar{H}(t)) - \bar{f}(Q_\theta(t-1),H_\theta(t)) \\
    &\quad + \bar{f}(Q_\theta(t-1),H_\theta(t)) - f_\theta(Q_\theta(t-1),H_\theta(t)) \|_{\infty} \\
    \leq & \|  \bar{f}(\bar{Q}(t-1),\bar{H}(t)) - \bar{f}(Q_\theta(t-1),\bar{H}(t)) \|_{\infty} \\
    &\quad + \| \bar{f}(Q_\theta(t-1),\bar{H}(t)) - \bar{f}(Q_\theta(t-1),H_\theta(t)) \|_{\infty} \\
    &\quad + \| \bar{f}(Q_\theta(t-1),H_\theta(t)) - f_\theta(Q_\theta(t-1),H_\theta(t)) \|_{\infty} \\
    \leq &N \lambda_0 B  + N \epsilon_s B + N \epsilon_1 \\
    :=& \lambda_1(\epsilon_s, \epsilon_1).
\end{align*}


The above reasoning can then be applied recursively to prove that 
\begin{equation*}
 \|   \bar{Q}(t) - Q_\theta(t) \|_{\infty} \leq \lambda_t(\epsilon_s, \epsilon_1),
\end{equation*}
where $\lambda_t(\epsilon_s, \epsilon_1)$ could be found as little as possible, according to $\epsilon_s, \epsilon_1$.
Finally, let $\epsilon_2 > 0$, so that 
\begin{align*}
    & \| \bar{\varphi}(t,G,v) - \varphi_\theta(t,G,v) \|_{\infty} \\ & 
    =\| \overline{\text{READOUT}}_{\text{dyn}}( \bar{\mathbf{Q}}(t))-\text{READOUT}_{\text{dyn},\theta}( \mathbf{Q}_{\theta}(t))\|_\infty \\
    &\leq \| \overline{\text{READOUT}}_{\text{dyn}}( \bar{\mathbf{Q}}(t))- \overline{\text{READOUT}}_{\text{dyn}}( \mathbf{Q}_{\theta}(t)) \|_{\infty} \\ 
    &\, + \| \overline{\text{READOUT}}_{\text{dyn}}( \mathbf{Q}_{\theta}(t)) - \text{READOUT}_{\text{dyn},\theta}( \mathbf{Q}_{\theta}(t))\|_\infty \\
    &\leq \lambda_t B + \epsilon_2 = \lambda(\epsilon_s, \epsilon_1, \epsilon_2). \; 
\end{align*}
Thus, we choose $\epsilon_s, \epsilon_1, \epsilon_2$, s.t. $\lambda \leq \frac{\varepsilon}{2}$; going back in probability, we obtain
\begin{equation*}
P( \| \bar{\varphi}(t, G,v) - \varphi_\theta(t, G,v) \| \leq \frac{\varepsilon}{2}) \geq 1- \lambda \quad \quad \quad \forall \; t \in I,
\end{equation*} 
which, along with Eq.~\eqref{prob_partial}, proves the result.
\end{proof}
\bibliographystyle{model1-num-names}

\bibliography{bibliography}

\begin{thebibliography}{53}
\expandafter\ifx\csname natexlab\endcsname\relax\def\natexlab#1{#1}\fi
\providecommand{\url}[1]{\texttt{#1}}
\providecommand{\href}[2]{#2}
\providecommand{\path}[1]{#1}
\providecommand{\DOIprefix}{doi:}
\providecommand{\ArXivprefix}{arXiv:}
\providecommand{\URLprefix}{URL: }
\providecommand{\Pubmedprefix}{pmid:}
\providecommand{\doi}[1]{\href{http://dx.doi.org/#1}{\path{#1}}}
\providecommand{\Pubmed}[1]{\href{pmid:#1}{\path{#1}}}
\providecommand{\bibinfo}[2]{#2}
\ifx\xfnm\relax \def\xfnm[#1]{\unskip,\space#1}\fi
\bibitem[{Skardinga et~al.(2021)Skardinga, Gabrys, and Musial}]{skardinga2021foundations}
\bibinfo{author}{J.~Skardinga}, \bibinfo{author}{B.~Gabrys}, \bibinfo{author}{K.~Musial},
\newblock \bibinfo{title}{{Foundations and Modelling of Dynamic Networks Using Dynamic Graph Neural Networks: A Survey}},
\newblock \bibinfo{journal}{IEEE Access}  (\bibinfo{year}{2021}).
\bibitem[{Kazemi et~al.(2020)Kazemi, Goel, Jain, Kobyzev, Sethi, Forsyth, and Poupart}]{kazemi_survey_dyn_gnn}
\bibinfo{author}{S.~M. Kazemi}, \bibinfo{author}{R.~Goel}, \bibinfo{author}{K.~Jain}, \bibinfo{author}{I.~Kobyzev}, \bibinfo{author}{A.~Sethi}, \bibinfo{author}{P.~Forsyth}, \bibinfo{author}{P.~Poupart},
\newblock \bibinfo{title}{{Representation Learning for Dynamic Graphs: A Survey}},
\newblock \bibinfo{journal}{J. Mach. Learn. Res.} \bibinfo{volume}{21} (\bibinfo{year}{2020}) \bibinfo{pages}{70:1--70:73}.
\bibitem[{Thomas et~al.(2021)Thomas, Beddar-Wiesing, and Moallemy-Oureh}]{thomas2021graph}
\bibinfo{author}{J.~M. Thomas}, \bibinfo{author}{S.~Beddar-Wiesing}, \bibinfo{author}{A.~Moallemy-Oureh},
\newblock \bibinfo{title}{{Graph Type Expressivity and Transformations}},
\newblock \bibinfo{journal}{arXiv:2109.10708}  (\bibinfo{year}{2021}).
\bibitem[{Scarselli et~al.(2009)Scarselli, Gori, Tsoi, Hagenbuchner, and Monfardini}]{jour_scarselli_2009}
\bibinfo{author}{F.~Scarselli}, \bibinfo{author}{M.~Gori}, \bibinfo{author}{A.~C. Tsoi}, \bibinfo{author}{M.~Hagenbuchner}, \bibinfo{author}{G.~Monfardini},
\newblock \bibinfo{title}{{The Graph Neural Network Model}},
\newblock \bibinfo{journal}{IEEE Transactions on Neural Networks} \bibinfo{volume}{20} (\bibinfo{year}{2009}) \bibinfo{pages}{61--80}.
\bibitem[{Micheli(2009)}]{micheli2009neural}
\bibinfo{author}{A.~Micheli},
\newblock \bibinfo{title}{Neural network for graphs: A contextual constructive approach},
\newblock \bibinfo{journal}{IEEE Transactions on Neural Networks} \bibinfo{volume}{20} (\bibinfo{year}{2009}) \bibinfo{pages}{498--511}.
\bibitem[{Li et~al.(2016)Li, Tarlow, Brockschmidt, and Zemel}]{li2015gated}
\bibinfo{author}{Y.~Li}, \bibinfo{author}{D.~Tarlow}, \bibinfo{author}{M.~Brockschmidt}, \bibinfo{author}{R.~Zemel},
\newblock \bibinfo{title}{Gated graph sequence neural networks},
\newblock \bibinfo{journal}{4th International Conference on Learning Representations, {ICLR} 2016, San Juan, Puerto Rico, May 2-4, 2016, Conference Track Proceedings}  (\bibinfo{year}{2016}).
\bibitem[{Bruna et~al.(2014)Bruna, Zaremba, Szlam, and LeCun}]{bruna2013}
\bibinfo{author}{J.~Bruna}, \bibinfo{author}{W.~Zaremba}, \bibinfo{author}{A.~Szlam}, \bibinfo{author}{Y.~LeCun},
\newblock \bibinfo{title}{Spectral networks and locally connected networks on graphs},
\newblock \bibinfo{journal}{ICLR 2014}  (\bibinfo{year}{2014}).
\bibitem[{Kipf and Welling(2017)}]{kipf2016}
\bibinfo{author}{T.~N. Kipf}, \bibinfo{author}{M.~Welling},
\newblock \bibinfo{title}{Semi-supervised classification with graph convolutional networks},
\newblock \bibinfo{journal}{ICLR 2017}  (\bibinfo{year}{2017}).
\bibitem[{Hamilton et~al.(2017)Hamilton, Ying, and Leskovec}]{hamilton2017inductive}
\bibinfo{author}{W.~Hamilton}, \bibinfo{author}{Z.~Ying}, \bibinfo{author}{J.~Leskovec},
\newblock \bibinfo{title}{Inductive representation learning on large graphs},
\newblock in: \bibinfo{booktitle}{Advances in Neural Information Processing Systems}, \bibinfo{year}{2017}, pp. \bibinfo{pages}{1024--1034}.
\bibitem[{Veli{\v{c}}kovi{\'c} et~al.(2018)Veli{\v{c}}kovi{\'c}, Cucurull, Casanova, Romero, Li{\`{o}}, and Bengio}]{GAT}
\bibinfo{author}{P.~Veli{\v{c}}kovi{\'c}}, \bibinfo{author}{G.~Cucurull}, \bibinfo{author}{A.~Casanova}, \bibinfo{author}{A.~Romero}, \bibinfo{author}{P.~Li{\`{o}}}, \bibinfo{author}{Y.~Bengio},
\newblock \bibinfo{title}{Graph attention networks},
\newblock \bibinfo{journal}{ICLR 2018}  (\bibinfo{year}{2018}).
\bibitem[{Battaglia et~al.(2018)Battaglia, Hamrick, Bapst, Sanchez{-}Gonzalez, Zambaldi, Malinowski, Tacchetti, Raposo, Santoro, Faulkner, G{\"{u}}l{\c{c}}ehre, Song, Ballard, Gilmer, Dahl, Vaswani, Allen, Nash, Langston, Dyer, Heess, Wierstra, Kohli, Botvinick, Vinyals, Li, and Pascanu}]{Battaglia2018}
\bibinfo{author}{P.~W. Battaglia}, \bibinfo{author}{J.~B. Hamrick}, \bibinfo{author}{V.~Bapst}, \bibinfo{author}{A.~Sanchez{-}Gonzalez}, \bibinfo{author}{V.~F. Zambaldi}, \bibinfo{author}{M.~Malinowski}, \bibinfo{author}{A.~Tacchetti}, \bibinfo{author}{D.~Raposo}, \bibinfo{author}{A.~Santoro}, \bibinfo{author}{R.~Faulkner}, \bibinfo{author}{{\c{C}}.~G{\"{u}}l{\c{c}}ehre}, \bibinfo{author}{H.~F. Song}, \bibinfo{author}{A.~J. Ballard}, \bibinfo{author}{J.~Gilmer}, \bibinfo{author}{G.~E. Dahl}, \bibinfo{author}{A.~Vaswani}, \bibinfo{author}{K.~R. Allen}, \bibinfo{author}{C.~Nash}, \bibinfo{author}{V.~Langston}, \bibinfo{author}{C.~Dyer}, \bibinfo{author}{N.~Heess}, \bibinfo{author}{D.~Wierstra}, \bibinfo{author}{P.~Kohli}, \bibinfo{author}{M.~M. Botvinick}, \bibinfo{author}{O.~Vinyals}, \bibinfo{author}{Y.~Li}, \bibinfo{author}{R.~Pascanu},
\newblock \bibinfo{title}{Relational inductive biases, deep learning, and graph networks},
\newblock \bibinfo{journal}{arXiv preprint arXiv:1806.01261}  (\bibinfo{year}{2018}).
\bibitem[{Jegelka(2022)}]{jegelka2022theory}
\bibinfo{author}{S.~Jegelka},
\newblock \bibinfo{title}{Theory of graph neural networks: Representation and learning},
\newblock \bibinfo{journal}{arXiv preprint arXiv:2204.07697}  (\bibinfo{year}{2022}).
\bibitem[{Zhou et~al.(2020)Zhou, Cui, Hu, Zhang, Yang, Liu, Wang, Li, and Sun}]{zhou2020graph}
\bibinfo{author}{J.~Zhou}, \bibinfo{author}{G.~Cui}, \bibinfo{author}{S.~Hu}, \bibinfo{author}{Z.~Zhang}, \bibinfo{author}{C.~Yang}, \bibinfo{author}{Z.~Liu}, \bibinfo{author}{L.~Wang}, \bibinfo{author}{C.~Li}, \bibinfo{author}{M.~Sun},
\newblock \bibinfo{title}{Graph neural networks: A review of methods and applications},
\newblock \bibinfo{journal}{AI open} \bibinfo{volume}{1} (\bibinfo{year}{2020}) \bibinfo{pages}{57--81}.
\bibitem[{Luo et~al.(2021)Luo, Shi, Xu, and Tang}]{luo2021predicting}
\bibinfo{author}{S.~Luo}, \bibinfo{author}{C.~Shi}, \bibinfo{author}{M.~Xu}, \bibinfo{author}{J.~Tang},
\newblock \bibinfo{title}{Predicting molecular conformation via dynamic graph score matching},
\newblock \bibinfo{journal}{Advances in Neural Information Processing Systems} \bibinfo{volume}{34} (\bibinfo{year}{2021}) \bibinfo{pages}{19784--19795}.
\bibitem[{Deng et~al.(2019)Deng, Rangwala, and Ning}]{deng2019learning}
\bibinfo{author}{S.~Deng}, \bibinfo{author}{H.~Rangwala}, \bibinfo{author}{Y.~Ning},
\newblock \bibinfo{title}{Learning dynamic context graphs for predicting social events},
\newblock in: \bibinfo{booktitle}{Proceedings of the 25th ACM SIGKDD International Conference on Knowledge Discovery \& Data Mining}, \bibinfo{year}{2019}, pp. \bibinfo{pages}{1007--1016}.
\bibitem[{Leman and Weisfeiler(1968)}]{leman1968}
\bibinfo{author}{A.~Leman}, \bibinfo{author}{B.~Weisfeiler},
\newblock \bibinfo{title}{A reduction of a graph to a canonical form and an algebra arising during this reduction},
\newblock \bibinfo{journal}{Nauchno-Technicheskaya Informatsiya} \bibinfo{volume}{2} (\bibinfo{year}{1968}) \bibinfo{pages}{12--16}.
\bibitem[{Shervashidze et~al.(2011)Shervashidze, Schweitzer, Van~Leeuwen, Mehlhorn, and Borgwardt}]{wltest2011}
\bibinfo{author}{N.~Shervashidze}, \bibinfo{author}{P.~Schweitzer}, \bibinfo{author}{E.~J. Van~Leeuwen}, \bibinfo{author}{K.~Mehlhorn}, \bibinfo{author}{K.~M. Borgwardt},
\newblock \bibinfo{title}{{Weisfeiler-Lehman Graph Kernels}},
\newblock \bibinfo{journal}{Journal of Machine Learning Research}  (\bibinfo{year}{2011}).
\bibitem[{Xu et~al.(2019)Xu, Hu, Leskovec, and Jegelka}]{xu2018powerful}
\bibinfo{author}{K.~Xu}, \bibinfo{author}{W.~Hu}, \bibinfo{author}{J.~Leskovec}, \bibinfo{author}{S.~Jegelka},
\newblock \bibinfo{title}{How powerful are graph neural networks?},
\newblock \bibinfo{journal}{7th International Conference on Learning Representations, {ICLR} 2019, New Orleans, LA, USA, May 6-9, 2019}  (\bibinfo{year}{2019}).
\bibitem[{Krebs and Verbitsky(2015)}]{krebs2015universal}
\bibinfo{author}{A.~Krebs}, \bibinfo{author}{O.~Verbitsky},
\newblock \bibinfo{title}{Universal covers, color refinement, and two-variable counting logic: Lower bounds for the depth},
\newblock in: \bibinfo{booktitle}{2015 30th Annual ACM/IEEE Symposium on Logic in Computer Science}, \bibinfo{organization}{IEEE}, \bibinfo{year}{2015}, pp. \bibinfo{pages}{689--700}.
\bibitem[{{D'Inverno, Giuseppe A. and Bianchini, Monica and Sampoli, Maria L. and Scarselli, Franco}(2021)}]{d2021unifying}
\bibinfo{author}{{D'Inverno, Giuseppe A. and Bianchini, Monica and Sampoli, Maria L. and Scarselli, Franco}},
\newblock \bibinfo{title}{{On the approximation capability of GNNs in node classification/regression tasks}},
\newblock \bibinfo{journal}{arXiv preprint arXiv:2106.08992}  (\bibinfo{year}{2021}).
\bibitem[{Garg et~al.(2020)Garg, Jegelka, and Jaakkola}]{garg2020generalization}
\bibinfo{author}{V.~Garg}, \bibinfo{author}{S.~Jegelka}, \bibinfo{author}{T.~Jaakkola},
\newblock \bibinfo{title}{Generalization and representational limits of graph neural networks},
\newblock in: \bibinfo{booktitle}{International Conference on Machine Learning}, \bibinfo{organization}{PMLR}, \bibinfo{year}{2020}, pp. \bibinfo{pages}{3419--3430}.
\bibitem[{Scarselli et~al.(2008)Scarselli, Gori, Tsoi, Hagenbuchner, and Monfardini}]{scarselli2008computational}
\bibinfo{author}{F.~Scarselli}, \bibinfo{author}{M.~Gori}, \bibinfo{author}{A.~C. Tsoi}, \bibinfo{author}{M.~Hagenbuchner}, \bibinfo{author}{G.~Monfardini},
\newblock \bibinfo{title}{{Computational Capabilities of Graph Neural Networks}},
\newblock \bibinfo{journal}{IEEE Transactions on Neural Networks} \bibinfo{volume}{20} (\bibinfo{year}{2008}) \bibinfo{pages}{81--102}.
\bibitem[{Xu et~al.(2020)Xu, Ruan, Korpeoglu, Kumar, and Achan}]{xu2020inductive}
\bibinfo{author}{D.~Xu}, \bibinfo{author}{C.~Ruan}, \bibinfo{author}{E.~Korpeoglu}, \bibinfo{author}{S.~Kumar}, \bibinfo{author}{K.~Achan},
\newblock \bibinfo{title}{Inductive representation learning on temporal graphs},
\newblock \bibinfo{journal}{arXiv preprint arXiv:2002.07962}  (\bibinfo{year}{2020}).
\bibitem[{Rossi et~al.(2020)Rossi, Chamberlain, Frasca, Eynard, Monti, and Bronstein}]{rossi2020temporal}
\bibinfo{author}{E.~Rossi}, \bibinfo{author}{B.~Chamberlain}, \bibinfo{author}{F.~Frasca}, \bibinfo{author}{D.~Eynard}, \bibinfo{author}{F.~Monti}, \bibinfo{author}{M.~Bronstein},
\newblock \bibinfo{title}{Temporal graph networks for deep learning on dynamic graphs},
\newblock \bibinfo{journal}{arXiv preprint arXiv:2006.10637}  (\bibinfo{year}{2020}).
\bibitem[{Yu et~al.(2017)Yu, Yin, and Zhu}]{yu2017spatio}
\bibinfo{author}{B.~Yu}, \bibinfo{author}{H.~Yin}, \bibinfo{author}{Z.~Zhu},
\newblock \bibinfo{title}{Spatio-temporal graph convolutional networks: A deep learning framework for traffic forecasting},
\newblock \bibinfo{journal}{arXiv preprint arXiv:1709.04875}  (\bibinfo{year}{2017}).
\bibitem[{Trivedi et~al.(2017)Trivedi, Dai, Wang, and Song}]{trivedi2017know}
\bibinfo{author}{R.~Trivedi}, \bibinfo{author}{H.~Dai}, \bibinfo{author}{Y.~Wang}, \bibinfo{author}{L.~Song},
\newblock \bibinfo{title}{Know-evolve: Deep temporal reasoning for dynamic knowledge graphs},
\newblock in: \bibinfo{booktitle}{international conference on machine learning}, \bibinfo{organization}{PMLR}, \bibinfo{year}{2017}, pp. \bibinfo{pages}{3462--3471}.
\bibitem[{Wu et~al.(2020)Wu, Cao, Cheung, and Hamilton}]{wu2020temp}
\bibinfo{author}{J.~Wu}, \bibinfo{author}{M.~Cao}, \bibinfo{author}{J.~C.~K. Cheung}, \bibinfo{author}{W.~L. Hamilton},
\newblock \bibinfo{title}{Temp: Temporal message passing for temporal knowledge graph completion},
\newblock \bibinfo{journal}{arXiv preprint arXiv:2010.03526}  (\bibinfo{year}{2020}).
\bibitem[{Thomas et~al.(2022)Thomas, Moallemy-Oureh, Beddar-Wiesing, and Holzh{\"u}ter}]{thomas2022graph}
\bibinfo{author}{J.~Thomas}, \bibinfo{author}{A.~Moallemy-Oureh}, \bibinfo{author}{S.~Beddar-Wiesing}, \bibinfo{author}{C.~Holzh{\"u}ter},
\newblock \bibinfo{title}{Graph neural networks designed for different graph types: A survey},
\newblock \bibinfo{journal}{Transactions on Machine Learning Research}  (\bibinfo{year}{2022}).
\bibitem[{Barros et~al.(2021)Barros, Mendon{\c{c}}a, Vieira, and Ziviani}]{barros2021survey}
\bibinfo{author}{C.~D. Barros}, \bibinfo{author}{M.~R. Mendon{\c{c}}a}, \bibinfo{author}{A.~B. Vieira}, \bibinfo{author}{A.~Ziviani},
\newblock \bibinfo{title}{A survey on embedding dynamic graphs},
\newblock \bibinfo{journal}{ACM Computing Surveys (CSUR)} \bibinfo{volume}{55} (\bibinfo{year}{2021}) \bibinfo{pages}{1--37}.
\bibitem[{Longa et~al.(2023)Longa, Lachi, Santin, Bianchini, Lepri, Lio, franco scarselli, and Passerini}]{longa2023graph}
\bibinfo{author}{A.~Longa}, \bibinfo{author}{V.~Lachi}, \bibinfo{author}{G.~Santin}, \bibinfo{author}{M.~Bianchini}, \bibinfo{author}{B.~Lepri}, \bibinfo{author}{P.~Lio}, \bibinfo{author}{franco scarselli}, \bibinfo{author}{A.~Passerini},
\newblock \bibinfo{title}{Graph neural networks for temporal graphs: State of the art, open challenges, and opportunities},
\newblock \bibinfo{journal}{Transactions on Machine Learning Research}  (\bibinfo{year}{2023}).
\bibitem[{Xue et~al.(2022)Xue, Zhong, Li, Chen, Zhai, and Kong}]{xue2022dynamic}
\bibinfo{author}{G.~Xue}, \bibinfo{author}{M.~Zhong}, \bibinfo{author}{J.~Li}, \bibinfo{author}{J.~Chen}, \bibinfo{author}{C.~Zhai}, \bibinfo{author}{R.~Kong},
\newblock \bibinfo{title}{Dynamic network embedding survey},
\newblock \bibinfo{journal}{Neurocomputing} \bibinfo{volume}{472} (\bibinfo{year}{2022}) \bibinfo{pages}{212--223}.
\bibitem[{You et~al.(2021)You, Gomes-Selman, Ying, and Leskovec}]{you2021identity}
\bibinfo{author}{J.~You}, \bibinfo{author}{J.~Gomes-Selman}, \bibinfo{author}{R.~Ying}, \bibinfo{author}{J.~Leskovec},
\newblock \bibinfo{title}{Identity-aware graph neural networks},
\newblock \bibinfo{journal}{Thirty-Fifth {AAAI} Conference on Artificial Intelligence, {AAAI} 2021, Thirty-Third Conference on Innovative Applications of Artificial Intelligence, {IAAI} 2021, The Eleventh Symposium on Educational Advances in Artificial Intelligence, {EAAI} 2021, Virtual Event, February 2-9, 2021}  (\bibinfo{year}{2021}) \bibinfo{pages}{10737--10745}.
\bibitem[{Morris et~al.(2019)Morris, Ritzert, Fey, Hamilton, Lenssen, Rattan, and Grohe}]{morris_2019_WL_go_neural}
\bibinfo{author}{C.~Morris}, \bibinfo{author}{M.~Ritzert}, \bibinfo{author}{M.~Fey}, \bibinfo{author}{W.~L. Hamilton}, \bibinfo{author}{J.~E. Lenssen}, \bibinfo{author}{G.~Rattan}, \bibinfo{author}{M.~Grohe},
\newblock \bibinfo{title}{{Weisfeiler and Leman Go Neural: Higher-Order Graph Neural Networks}},
\newblock in: \bibinfo{booktitle}{The Thirty-Third {AAAI} Conference on Artificial Intelligence, {AAAI} 2019, The Thirty-First Innovative Applications of Artificial Intelligence Conference, {IAAI} 2019, The Ninth {AAAI} Symposium on Educational Advances in Artificial Intelligence, {EAAI} 2019, Honolulu, Hawaii, USA, January 27 - February 1, 2019}, \bibinfo{publisher}{{AAAI} Press}, \bibinfo{year}{2019}, pp. \bibinfo{pages}{4602--4609}. \URLprefix \url{https://doi.org/10.1609/aaai.v33i01.33014602}. \DOIprefix\doi{10.1609/aaai.v33i01.33014602}.
\bibitem[{Bodnar et~al.(2021{\natexlab{a}})Bodnar, Frasca, Wang, Otter, Montufar, Lio, and Bronstein}]{bodnar2021weisfeiler}
\bibinfo{author}{C.~Bodnar}, \bibinfo{author}{F.~Frasca}, \bibinfo{author}{Y.~Wang}, \bibinfo{author}{N.~Otter}, \bibinfo{author}{G.~F. Montufar}, \bibinfo{author}{P.~Lio}, \bibinfo{author}{M.~Bronstein},
\newblock \bibinfo{title}{{Weisfeiler and Lehman Go Topological: Message Passing Simplicial Networks}},
\newblock in: \bibinfo{booktitle}{International Conference on Machine Learning}, \bibinfo{organization}{PMLR}, \bibinfo{year}{2021}{\natexlab{a}}, pp. \bibinfo{pages}{1026--1037}.
\bibitem[{Bodnar et~al.(2021{\natexlab{b}})Bodnar, Frasca, Otter, Wang, Li{\`o}, Montufar, and Bronstein}]{bodnar2021weisfeilerb}
\bibinfo{author}{C.~Bodnar}, \bibinfo{author}{F.~Frasca}, \bibinfo{author}{N.~Otter}, \bibinfo{author}{Y.~G. Wang}, \bibinfo{author}{P.~Li{\`o}}, \bibinfo{author}{G.~F. Montufar}, \bibinfo{author}{M.~Bronstein},
\newblock \bibinfo{title}{{Weisfeiler and Lehman Go Cellular: CW Networks}},
\newblock \bibinfo{journal}{Advances in Neural Information Processing Systems} \bibinfo{volume}{34} (\bibinfo{year}{2021}{\natexlab{b}}).
\bibitem[{Barcelo et~al.(2022)Barcelo, Galkin, Morris, and Orth}]{barcelo2022weisfeiler}
\bibinfo{author}{P.~Barcelo}, \bibinfo{author}{M.~Galkin}, \bibinfo{author}{C.~Morris}, \bibinfo{author}{M.~R. Orth},
\newblock \bibinfo{title}{Weisfeiler and leman go relational},
\newblock in: \bibinfo{booktitle}{Learning on Graphs Conference}, \bibinfo{organization}{PMLR}, \bibinfo{year}{2022}, pp. \bibinfo{pages}{46--1}.
\bibitem[{Sato(2020)}]{sato2020survey}
\bibinfo{author}{R.~Sato},
\newblock \bibinfo{title}{{A Survey on The Expressive Power of Graph Neural Networks}},
\newblock \bibinfo{journal}{ArXiv} \bibinfo{volume}{abs/2003.04078} (\bibinfo{year}{2020}).
\bibitem[{Zhang and Li(2021)}]{zhang2021nested}
\bibinfo{author}{M.~Zhang}, \bibinfo{author}{P.~Li},
\newblock \bibinfo{title}{Nested graph neural networks},
\newblock \bibinfo{journal}{Advances in Neural Information Processing Systems} \bibinfo{volume}{34} (\bibinfo{year}{2021}) \bibinfo{pages}{15734--15747}.
\bibitem[{Abboud et~al.(2021)Abboud, Ceylan, Grohe, and Lukasiewicz}]{abboud2020}
\bibinfo{author}{R.~Abboud}, \bibinfo{author}{{\.I}.~{\.I}. Ceylan}, \bibinfo{author}{M.~Grohe}, \bibinfo{author}{T.~Lukasiewicz},
\newblock \bibinfo{title}{The surprising power of graph neural networks with random node initialization},
\newblock \bibinfo{journal}{Proceedings of the Thirtieth International Joint Conference on Artificial Intelligence, {IJCAI} 2021, Virtual Event / Montreal, Canada, 19-27 August 2021}  (\bibinfo{year}{2021}) \bibinfo{pages}{2112--2118}.
\bibitem[{Dasoulas et~al.(2020)Dasoulas, Santos, Scaman, and Virmaux}]{dasoulas2019coloring}
\bibinfo{author}{G.~Dasoulas}, \bibinfo{author}{L.~D. Santos}, \bibinfo{author}{K.~Scaman}, \bibinfo{author}{A.~Virmaux},
\newblock \bibinfo{title}{Coloring graph neural networks for node disambiguation},
\newblock in: \bibinfo{editor}{C.~Bessiere} (Ed.), \bibinfo{booktitle}{Proceedings of the Twenty-Ninth International Joint Conference on Artificial Intelligence, {IJCAI} 2020}, \bibinfo{publisher}{ijcai.org}, \bibinfo{year}{2020}, pp. \bibinfo{pages}{2126--2132}. \URLprefix \url{https://doi.org/10.24963/ijcai.2020/294}. \DOIprefix\doi{10.24963/ijcai.2020/294}.
\bibitem[{Maron et~al.(2019)Maron, Ben-Hamu, Serviansky, and Lipman}]{maron2019provably}
\bibinfo{author}{H.~Maron}, \bibinfo{author}{H.~Ben-Hamu}, \bibinfo{author}{H.~Serviansky}, \bibinfo{author}{Y.~Lipman},
\newblock \bibinfo{title}{Provably powerful graph networks},
\newblock \bibinfo{journal}{Advances in neural information processing systems} \bibinfo{volume}{32} (\bibinfo{year}{2019}).
\bibitem[{Azizian and Lelarge(2021)}]{azizian2020expressive}
\bibinfo{author}{W.~Azizian}, \bibinfo{author}{M.~Lelarge},
\newblock \bibinfo{title}{Expressive power of invariant and equivariant graph neural networks},
\newblock \bibinfo{journal}{9th International Conference on Learning Representations, {ICLR} 2021, Virtual Event, Austria, May 3-7, 2021}  (\bibinfo{year}{2021}).
\bibitem[{Maron et~al.(2019)Maron, Ben-Hamu, Shamir, and Lipman}]{maron2018invariant}
\bibinfo{author}{H.~Maron}, \bibinfo{author}{H.~Ben-Hamu}, \bibinfo{author}{N.~Shamir}, \bibinfo{author}{Y.~Lipman},
\newblock \bibinfo{title}{Invariant and equivariant graph networks},
\newblock \bibinfo{journal}{7th International Conference on Learning Representations, {ICLR} 2019, New Orleans, LA, USA, May 6-9, 2019}  (\bibinfo{year}{2019}).
\bibitem[{Loukas(2020)}]{loukas2019graph}
\bibinfo{author}{A.~Loukas},
\newblock \bibinfo{title}{What graph neural networks cannot learn: depth vs width},
\newblock \bibinfo{journal}{8th International Conference on Learning Representations, {ICLR} 2020, Addis Ababa, Ethiopia, April 26-30, 2020}  (\bibinfo{year}{2020}).
\bibitem[{Grohe and Schweitzer(2020)}]{grohe2020graph}
\bibinfo{author}{M.~Grohe}, \bibinfo{author}{P.~Schweitzer},
\newblock \bibinfo{title}{{The Graph Isomorphism Problem}},
\newblock \bibinfo{journal}{Communications of the ACM} \bibinfo{volume}{63} (\bibinfo{year}{2020}) \bibinfo{pages}{128--134}.
\bibitem[{Bronstein et~al.(2021)Bronstein, Bruna, Cohen, and Velickovic}]{bronstein2021geometric}
\bibinfo{author}{M.~M. Bronstein}, \bibinfo{author}{J.~Bruna}, \bibinfo{author}{T.~Cohen}, \bibinfo{author}{P.~Velickovic},
\newblock \bibinfo{title}{{Geometric Deep Learning: Grids, Groups, Graphs, Geodesics, and Gauges}},
\newblock \bibinfo{journal}{CoRR} \bibinfo{volume}{abs/2104.13478} (\bibinfo{year}{2021}).
\bibitem[{Kiefer and McKay(2020)}]{kiefer2020iteration}
\bibinfo{author}{S.~Kiefer}, \bibinfo{author}{B.~D. McKay},
\newblock \bibinfo{title}{The iteration number of colour refinement},
\newblock \bibinfo{journal}{47th International Colloquium on Automata, Languages, and Programming (ICALP 2020)}  (\bibinfo{year}{2020}).
\bibitem[{Seo et~al.(2018)Seo, Defferrard, Vandergheynst, and Bresson}]{seo2018structured}
\bibinfo{author}{Y.~Seo}, \bibinfo{author}{M.~Defferrard}, \bibinfo{author}{P.~Vandergheynst}, \bibinfo{author}{X.~Bresson},
\newblock \bibinfo{title}{Structured sequence modeling with graph convolutional recurrent networks},
\newblock in: \bibinfo{booktitle}{International conference on neural information processing}, \bibinfo{organization}{Springer}, \bibinfo{year}{2018}, pp. \bibinfo{pages}{362--373}.
\bibitem[{Narayan and Roe(2018)}]{narayan2018learning}
\bibinfo{author}{A.~Narayan}, \bibinfo{author}{P.~H. Roe},
\newblock \bibinfo{title}{Learning graph dynamics using deep neural networks},
\newblock \bibinfo{journal}{IFAC-PapersOnLine} \bibinfo{volume}{51} (\bibinfo{year}{2018}) \bibinfo{pages}{433--438}.
\bibitem[{Niepert et~al.(2016)Niepert, Ahmed, and Kutzkov}]{niepert2016learning}
\bibinfo{author}{M.~Niepert}, \bibinfo{author}{M.~Ahmed}, \bibinfo{author}{K.~Kutzkov},
\newblock \bibinfo{title}{Learning convolutional neural networks for graphs},
\newblock in: \bibinfo{booktitle}{International conference on machine learning}, \bibinfo{organization}{PMLR}, \bibinfo{year}{2016}, pp. \bibinfo{pages}{2014--2023}.
\bibitem[{Taheri et~al.(2019)Taheri, Gimpel, and Berger-Wolf}]{taheri2019learning}
\bibinfo{author}{A.~Taheri}, \bibinfo{author}{K.~Gimpel}, \bibinfo{author}{T.~Berger-Wolf},
\newblock \bibinfo{title}{Learning to represent the evolution of dynamic graphs with recurrent models},
\newblock in: \bibinfo{booktitle}{Companion proceedings of the 2019 world wide web conference}, \bibinfo{year}{2019}, pp. \bibinfo{pages}{301--307}.
\bibitem[{Hammer(2000)}]{hammer2000approximation}
\bibinfo{author}{B.~Hammer},
\newblock \bibinfo{title}{On the approximation capability of recurrent neural networks},
\newblock \bibinfo{journal}{Neurocomputing} \bibinfo{volume}{31} (\bibinfo{year}{2000}) \bibinfo{pages}{107--123}.
\bibitem[{Hamilton et~al.(2017)Hamilton, Ying, and Leskovec}]{hamilton2017representation}
\bibinfo{author}{W.~L. Hamilton}, \bibinfo{author}{R.~Ying}, \bibinfo{author}{J.~Leskovec},
\newblock \bibinfo{title}{Representation learning on graphs: Methods and applications},
\newblock \bibinfo{journal}{arXiv preprint arXiv:1709.05584}  (\bibinfo{year}{2017}).

\end{thebibliography}

\printcredits

\bio{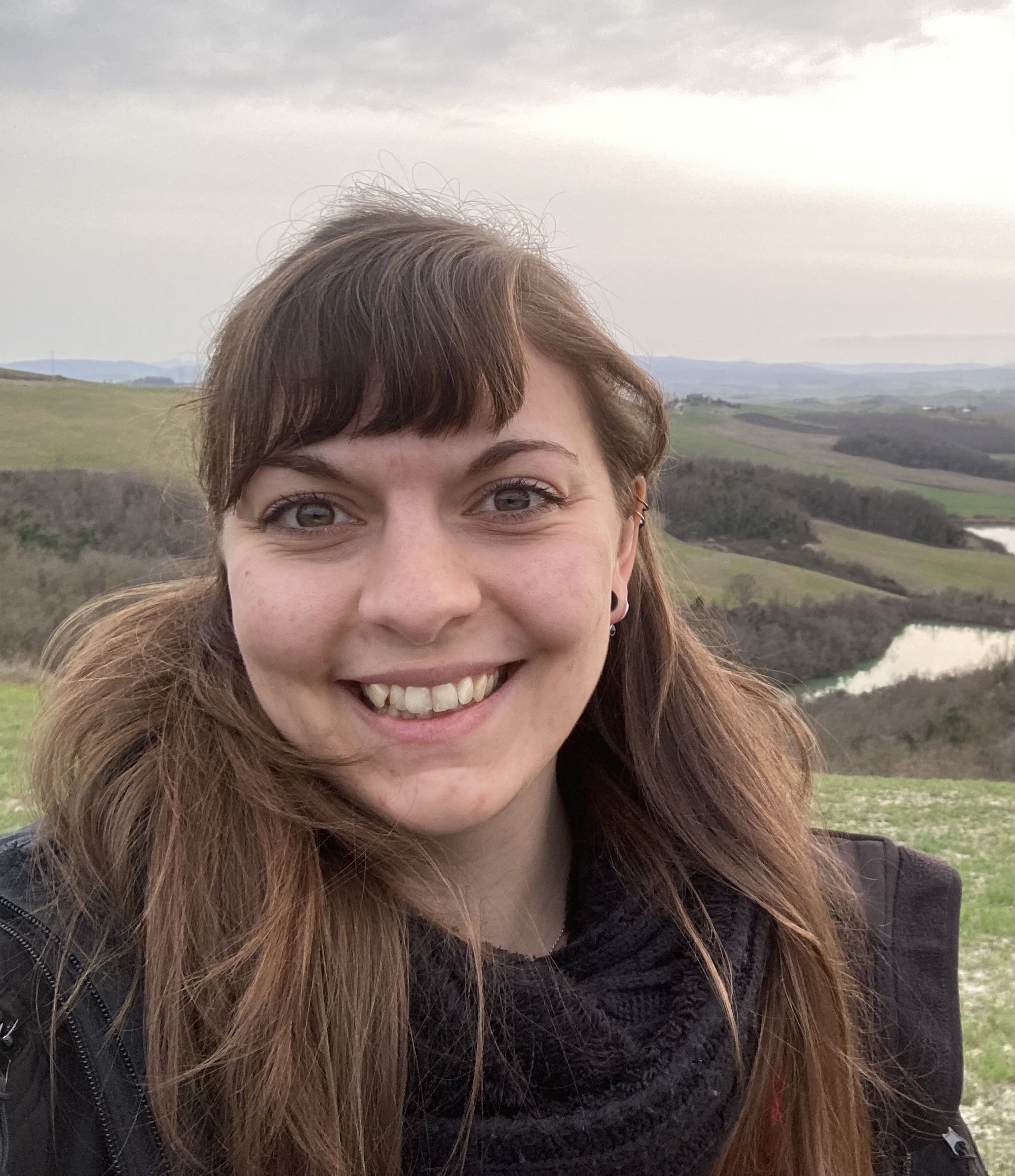}
Silvia Beddar-Wiesing is a Ph.D.~student in Computer Science at the University of Kassel. After her Bachelor's in Mathematics with a specialization in Mathematical Optimization in 2017 at the University of Kassel, she changed to Computer Science and absolved her Master's Degree with a specialization in Computational Intelligence and Data Analytics in 2020. In her research, she focuses on Machine Learning on structural-dynamic graphs, considering graph preprocessing, representation, and embedding techniques. Her further research interests cover topics from Timeseries Analysis, Graph Theory, Machine Learning, and Deep Neural Networks.
\endbio

\bio{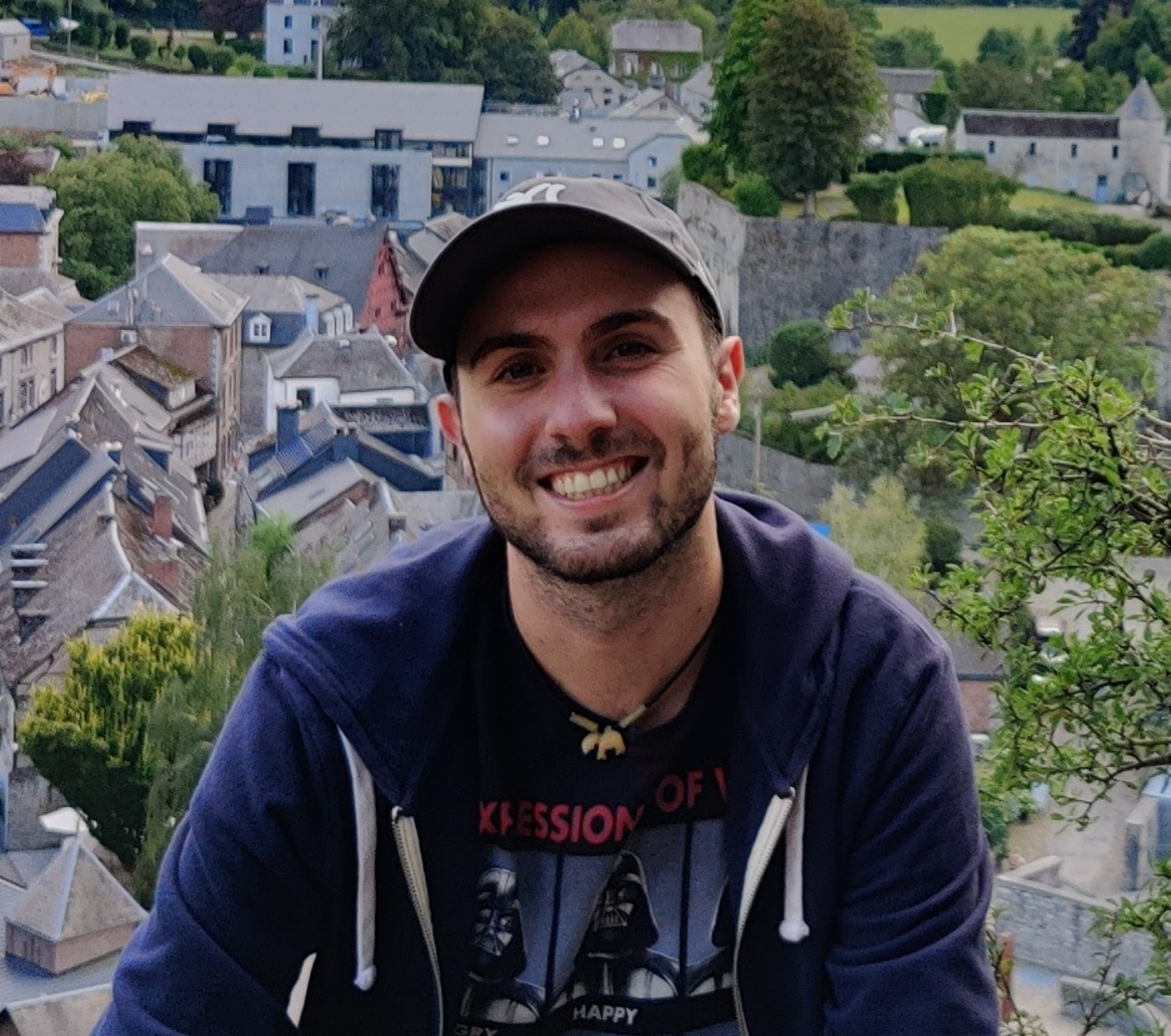}
Giuseppe Alessio D’Inverno is currently a Ph.D.~student in Information and Engineering Science at the University of Siena. He
took a Bachelor's Degree in Mathematics in 2018 and a Master's Degree cum
Laude in Applied Mathematics in 2020. His research interests are mainly
focused on the mathematical properties of Deep Neural Networks, the application of
Deep Learning to PDEs, graph theory, and geometric modeling.
\endbio
\bio{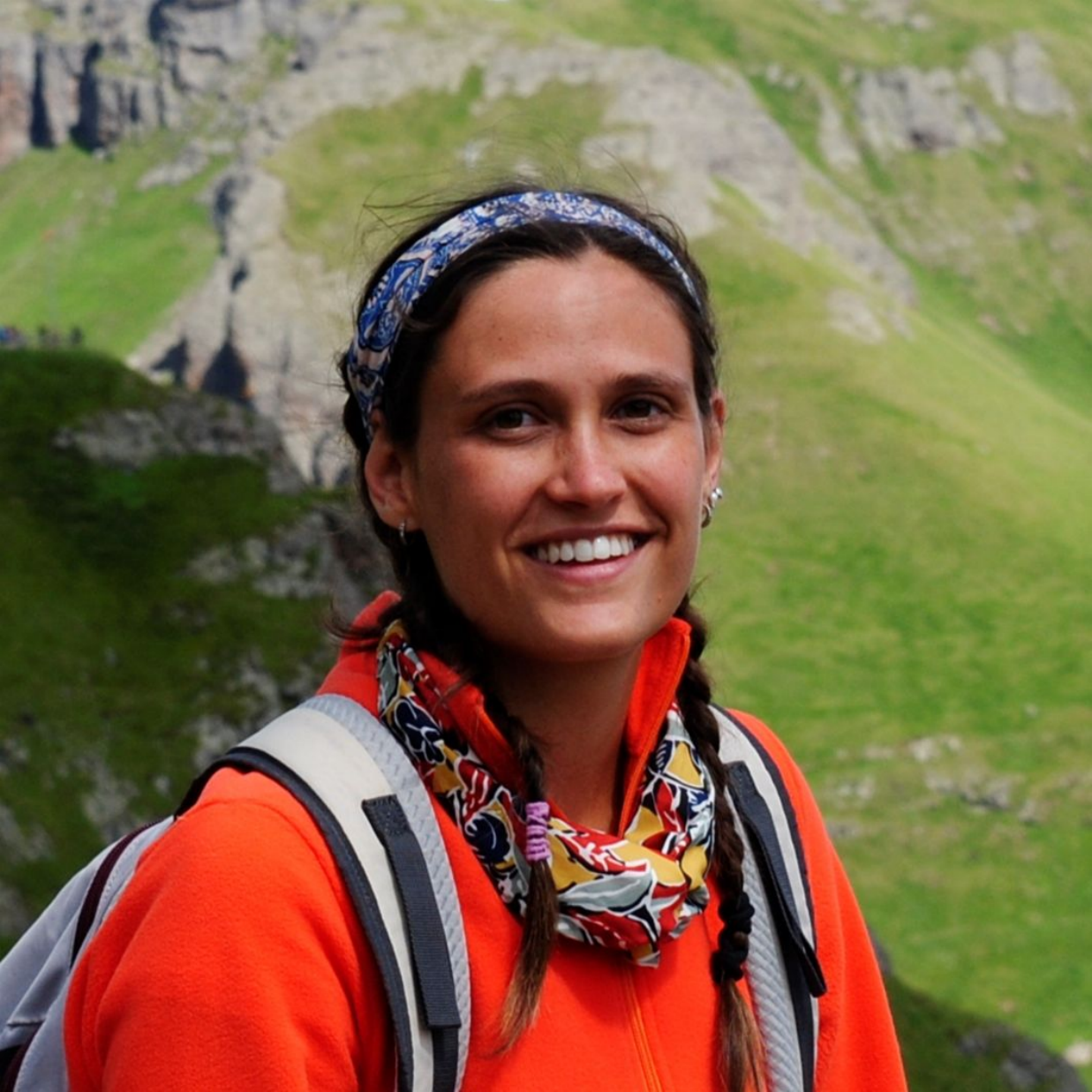}
Caterina Graziani is currently a Ph.D.~student in Information Engineering and Science at the University of Siena. In 2018 she received the Bachelor’s Degree Cum Laude in Mathematics from the University of Siena, and two years later, she obtained the Master’s Degree Cum Laude in Applied Mathematics at the University of Siena, defending a thesis called “LSTM for the prediction of translation speed based on Ribosome Profiling”. Her research interests are in Graph Theory, Graph Neural Networks, and Bioinformatics, with a particular interest in the mathematical foundations of Deep Neural Networks.
\endbio
\vspace{2cm}
\bio{figs/v2.pdf}
Veronica Lachi is a Ph.D.~student in Information and Engineering at the University of Siena. She graduated in Applied Mathematics with Honors in 2020; in 2018, she took a bachelor's degree in Economics with Honors. Her research interests mainly focus on Graph Theory, Graph Neural Networks, Community Detection in Graphs and Networks, and  mathematical properties of Deep Neural Networks.
\endbio
\bio{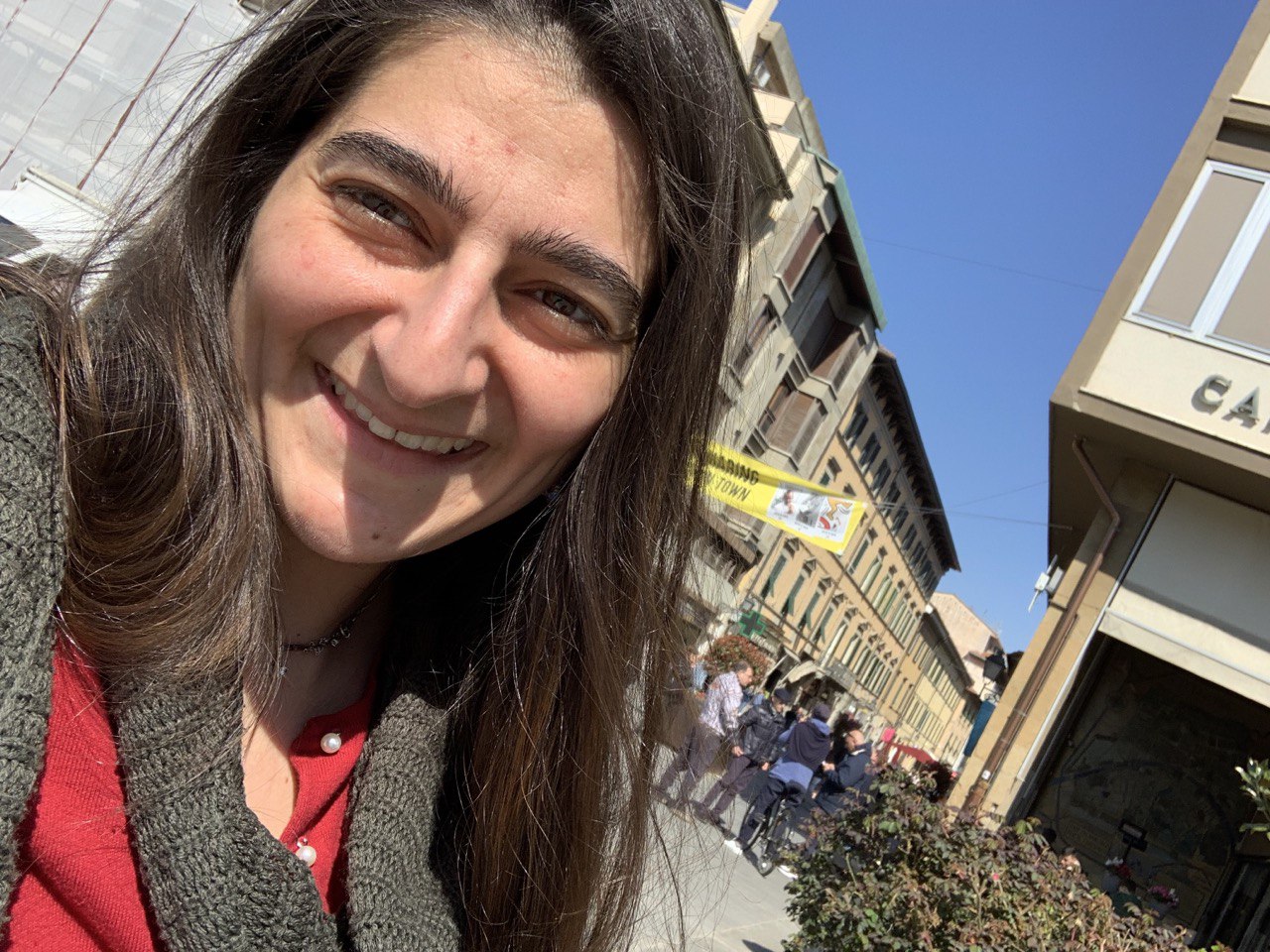}
Alice Moallemy-Oureh is currently a Ph.D.~candidate at the University of Kassel, Department of Electrical Engineering and Computer Science. 
Her topic addresses developing Graph Neural Networks designed to master the handling of attribute-dynamic graphs in continuous time representation.
Prior, she graduated with an M.Sc.~in Mathematics and its applications to Computer Science from the University of Kassel. Her specialization in Mathematics lies in Algorithmic Commutative Computer Algebra and Geometry. In Computer Science, she focused on Logic and Data Analytics.
Further interests cover topics from Graph Theory, Computer Algebra, IT-Security, Geometry, Pattern Recognition, and mathematical properties of Geometric Deep Learning.
\endbio
\bio{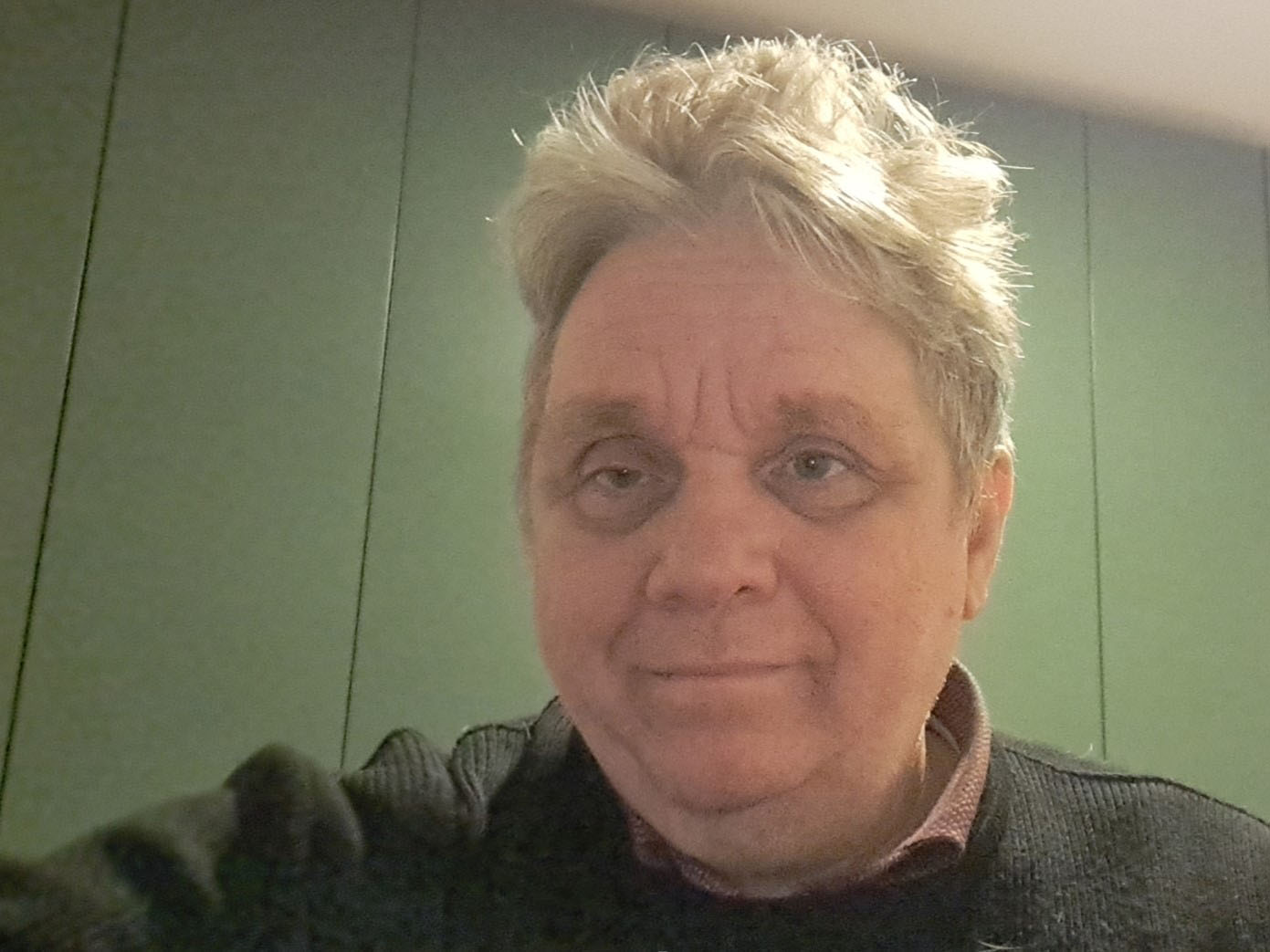}
Franco Scarselli received the Laurea degree with honors in Computer Science from the University of Pisa in 1989 and the Ph.D.~degree in Computer Science and Automation Engineering from the University of Florence in 1994. After the Ph.D.~he has been supported by foundations of private and public companies and by a postdoc of the University of Florence. In 1999, he moved to the University of Siena, where he was initially a research associate and, currently, an associate professor at the department of Information Engineering and Mathematics. 
Current theoretical research is mainly in machine learning, focusing on deep learning, graph neural networks,  and approximation theory. Research interests also include image understanding, information retrieval, and web applications.
\endbio
\bio{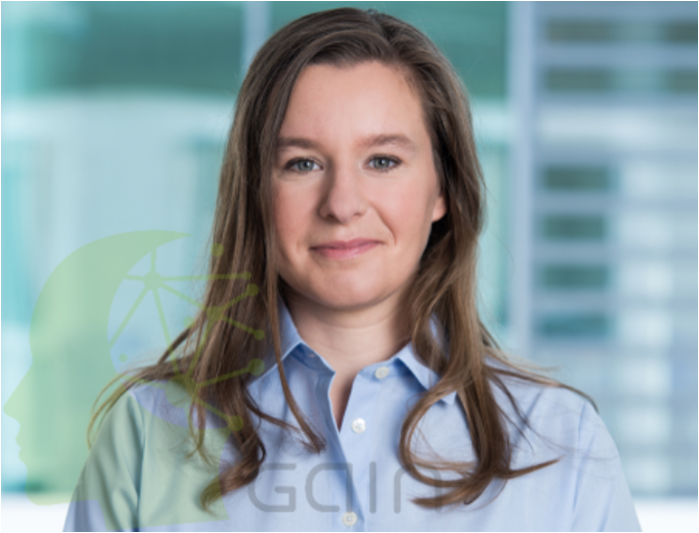}
Josephine Thomas received her degree in Theoretical Physics from TU Berlin. She gained her doctorate with a thesis on non-linear dimension reduction applied to a topic of complex network theory from TU Dresden.
Currently, she is the leader of the junior research group GAIN 'Graphs in Artificial Intelligence and Neural Networks at the University of Kassel. Her focus is on Graph Neural Networks for dynamic graphs. 
\endbio

\end{document}